\PassOptionsToPackage{capitalize,noabbrev}{cleveref}

\documentclass[]{preprint}

\usepackage[toc,page,header]{appendix}

\usepackage[utf8]{inputenc} 
\usepackage[T1]{fontenc}    
\usepackage{hyperref}       
\usepackage{url}            
\usepackage{booktabs}       
\usepackage{amsfonts}       
\usepackage{nicefrac}       
\usepackage{microtype}      
\usepackage{xcolor}         

\usepackage{amsmath}
\usepackage{amssymb}
\usepackage{mathtools}
\usepackage{amsthm}

\usepackage{hyperref}
\usepackage{url}
\usepackage{xcolor}
\usepackage{pifont}

\usepackage[utf8]{inputenc}
\usepackage{csquotes} 

\usepackage{booktabs}

\usepackage{enumitem}

\usepackage{booktabs}
\usepackage{multirow}
\usepackage{array}
\usepackage{colortbl}

\usepackage{arydshln} 

\usepackage{wrapfig}
\usepackage{subcaption}  
\usepackage{dblfloatfix}

\usepackage[most]{tcolorbox}
\usepackage{xcolor}
\usepackage{graphicx}

\definecolor{cvprblue}{RGB}{0,102,204} 

\usepackage{amsmath,bm}

\usepackage{color}
\usepackage{tikz}
\usetikzlibrary{shapes, arrows.meta, positioning, calc, fit, backgrounds}

\usepackage{etoc}

\definecolor{sybcblue}{HTML}{004488}  
\definecolor{brickred}{HTML}{BB5566}  
\definecolor{softgray}{HTML}{666666}  

\definecolor{graybg}{gray}{0.9} 

\definecolor{forestgreen}{rgb}{0.133,0.549,0.133}
\definecolor{crimson}{rgb}{0.863,0.078,0.235}
\newcommand{\cmark}{\textcolor{forestgreen!80!black}{\ding{51}}}
\newcommand{\xmark}{\textcolor{crimson!80!black}{\ding{55}}}

\definecolor{natureblue}{RGB}{0,76,153}
\newcommand{\fancynumber}[1]{%
\raisebox{1pt}{%
  \tikz[baseline=(char.base)]{
    \node[
      shape=circle,
      draw=black,
      fill=natureblue!20,
      inner sep=0pt,
      minimum size=0.8em,
      font=\tiny,
      text=black,
    ](char){#1};%
  }%
}%
}

\usepackage{hyperref}
\hypersetup{
    colorlinks=false,
    linkcolor=black,
    filecolor=magenta,      
    urlcolor=magenta, 
    pdftitle={VideoDR Benchmark},
}

\usepackage{tikz}

\usepackage[capitalize,noabbrev]{cleveref}

\theoremstyle{plain}
\newtheorem{theorem}{Theorem}[section]
\newtheorem{proposition}[theorem]{Proposition}
\newtheorem{lemma}[theorem]{Lemma}
\newtheorem{corollary}[theorem]{Corollary}
\theoremstyle{definition}
\newtheorem{definition}[theorem]{Definition}

\theoremstyle{remark}

\usepackage{tikz}
\usetikzlibrary{
    arrows.meta,
    calc,
    positioning,
    fit,
    backgrounds,
    decorations.pathreplacing,
    angles,
    quotes
}

\definecolor{Ublue}{RGB}{78,121,167}
\definecolor{Utgreen}{RGB}{88,167,107}
\definecolor{Vred}{RGB}{196,78,82}
\definecolor{GradOrange}{RGB}{242,142,43}
\definecolor{ResidualBlue}{RGB}{78,121,167}
\definecolor{AccentPurple}{RGB}{156,117,195}
\definecolor{wine}{HTML}{830E0D}

\usepackage{xcolor}
\definecolor{mitred}{RGB}{117,0,20}

\usepackage{tcolorbox}

\title{Modality Gap–Driven Subspace Alignment Training Paradigm For Multimodal Large Language Models}

\author[1,2]{Xiaomin Yu}
\author[3,4]{Yi Xin}
\author[5]{Yuhui Zhang}
\author[1]{Wenjie Zhang}
\author[6]{Chonghan Liu}

\makeatletter
\renewcommand\author[2][]{\addtolist[#1]{#2}{\authorlist}{\authorformat}{\\[0pt]}}
\makeatother

\author[2]{Hanzhen Zhao}

\makeatletter
\renewcommand\author[2][]{\addtolist[#1]{#2}{\authorlist}{\authorformat}{, }}
\makeatother

\author[7]{Chen Liu}
\author[8]{Xiaoxing Hu}
\author[9]{Ziyue Qiao}
\author[10]{Hao Tang}
\makeatletter
\renewcommand\author[2][]{\addtolist[#1]{#2}{\authorlist}{\authorformat}{\\[0pt]}}
\makeatother

\author[2]{Xiaobin Hu}

\makeatletter
\renewcommand\author[2][]{\addtolist[#1]{#2}{\authorlist}{\authorformat}{, }}
\makeatother
\author[1]{Chengwei Qin}
\author[1]{Hui Xiong}
\author[4]{Yu Qiao}
\author[2]{Shuicheng YAN}

\affiliation[1]{HKUST(GZ)}
\affiliation[2]{NUS}
\affiliation[3]{sh AILab}
\affiliation[4]{SII}
\affiliation[5]{Stanford}
\affiliation[6]{UCLA}
\affiliation[7]{Yale}
\affiliation[8]{SJTU}
\affiliation[9]{GBU}
\affiliation[10]{PKU}

\abstract{
Despite the success of multimodal contrastive learning in aligning visual and linguistic representations, a persistent geometric anomaly, the Modality Gap, remains: embeddings of distinct modalities expressing identical semantics occupy systematically offset regions. Prior approaches to bridge this gap are largely limited by oversimplified isotropic assumptions, hindering their application in large-scale scenarios. In this paper, we address these limitations by precisely characterizing the geometric shape of the modality gap and leveraging it for efficient model scaling. First, we propose the Fixed-frame Modality Gap Theory, which decomposes the modality gap within a frozen reference frame into stable biases and anisotropic residuals. Guided by this precise modeling, we introduce ReAlign, a training-free modality alignment strategy. Utilizing statistics from massive unpaired data, ReAlign aligns text representation into the image representation distribution via a three-step process comprising Anchor, Trace, and Centroid Alignment, thereby explicitly rectifying geometric misalignment. Building on ReAlign, we propose ReVision, a scalable training paradigm for Multimodal Large Language Models~(MLLMs). ReVision integrates ReAlign into the pretraining stage, enabling the model to learn the distribution of visual representations from unpaired text before visual instruction tuning, without the need for large-scale, high-quality image-text pairs. Our framework demonstrates that statistically aligned unpaired data can effectively substitute for expensive image-text pairs, offering a robust path for the efficient scaling of MLLMs.
}

\date{\today}
\checkdata[Leader]{Xiaomin Yu (\email{yuxm02@gmail.com})}
\correspondence{Xiaobin Hu, Chengwei Qin}

\checkdata[Github]{\url{https://github.com/Yu-xm/Modality_Gap_Theory.git}}

\begin{document}
\maketitle

\newpage

\section{Introduction}


Multimodal contrastive learning~\cite{huang2024llm2clip,zhai2023sigmoid,radford2021learning,sun2023eva,yaras2024explaining} has established itself as the standard paradigm for aligning visual and text representations. However, despite extensive training on massive image-text pairs, a persistent empirical observation remains: 


\begin{tcolorbox}[colframe=mitred, opacityback=0.9, arc=4pt, left=4pt, right=4pt, top=2pt, bottom=0pt,
title=Definition. Modality Gap Phenomenon]
\textit{For two modalities that express the same underlying semantics, their representations typically do not coincide. Instead, the two modalities inherently occupy distinct, systematically offset regions of the joint representation space. This phenomenon is called the \textbf{Modality Gap}~\cite{liang2022mind,zhang2024connect}}.
\end{tcolorbox}


While this gap hinders direct cross-modal interchangeability, it also suggests that, if the geometric misalignment can be effectively bridged, abundant text data could serve as a substitute~\cite{yu2025unicorn} for expensive image-text data in Multimodal Large Language Models (MLLMs) training.

Prior research has largely advanced along two directions, yet both face limitations: \textbf{\fancynumber{1} Geometric Correction:} These approaches attempt to post-hoc correct this gap via explicit geometric projections~\cite{zhang2024connect,yi2025decipher}. However, most of these works are limited to single, small-scale tasks~\cite{tewel2022zerocap,li2023decap,liu2024arcsin,su2022language,nukrai2022text} such as Image Captioning, failing to unlock the true potential of modality gap theory for model scaling. Furthermore, they typically rely on isotropic noise assumptions, failing to account for the complex anisotropic structures in high-dimensional spaces. \textbf{\fancynumber{2} Text-only Large-scale Training:} These methods~\cite{yu2025unicorn} leverage modality gap to synthesize pseudo-visual supervision signals from pure text. While promising, existing methods suffer from substantial performance degradation on fine-grained visual tasks, exposing a pronounced distributional gap between synthetic text representations and real-world image data.

These limitations point to a fundamental mismatch: existing assumptions are overly simplified, resulting in a lack of precise modeling of the modality gap's geometric shape, which in turn hinders its application in large-scale training. This mismatch naturally motivates two core research questions:
\begin{tcolorbox}[boxrule=1pt, colframe=black!70, colback=gray!5, left=8pt, right=8pt, top=4pt, bottom=4pt]

\textbf{\fancynumber{1} On Shape:} Can we move beyond simple mean assumptions to precisely characterize the intrinsic geometric shape of the modality gap within a stable reference frame?


\textbf{\fancynumber{2} On Scale:} Can we leverage this precise shape modeling to design a scalable training paradigm that substitutes expensive paired data with massive, easily accessible unpaired data, thereby achieving efficient MLLM scaling? 

\end{tcolorbox}

To address the first question regarding shape, we conduct the first empirical study that trains a contrastive dual-encoder~\cite{wu2023tinyclip} from scratch to precisely track and model the evolution of the modality gap. Based on this microscopic analysis, we propose a unified theoretical framework: we no longer treat the modality gap as random fluctuations but mathematically decompose it within a frozen reference frame ($\mathbb{R}^d = U \oplus V$). By explicitly separating the effective task subspace ($U$), where gradients concentrate and semantic information resides, from its orthogonal complement ($V$), we reveal the dual geometric structure of the modality gap: it comprises not only a stable bias component but also a residual component characterized by specific second-moment properties. This discovery allows us to transcend simple mean-based descriptions and fully capture the anisotropic distribution of the modality gap across different subspaces.

Building on this precise geometric modeling, we further address the second question regarding scale. Our approach highlights a perspective different from prior work: while acquiring high-quality paired image-text data~\cite{he2024efficient,chen2024sharegpt4v,li2024densefusion} is expensive, obtaining massive amounts of unpaired single-modality data is extremely easy. We believe that leveraging such large-scale unpaired data is sufficient to precisely reconstruct the shape of the modality gap via statistical laws, without relying on expensive paired samples. Guided by this insight, we introduce two core contributions:


\fancynumber{1} \textbf{ReAlign.} Building on the geometric analysis in Sec.~\ref{sec:modality_gap}, we introduce ReAlign, a training-free pre-alignment strategy that maps text representations into the image representation distribution using statistics derived from large-scale unpaired data. ReAlign operates through a three-stage procedure. First, Anchor Alignment matches first-order statistics (means). Second, Trace Alignment matches the global variance scale. Finally, Centroid Alignment explicitly rectifies the geometric drift induced by spherical projection. Together, these stages achieve precise cross-modal alignment at the statistical level using only linear transformations and normalization, without requiring any additional training.

\fancynumber{2} \textbf{ReVision.} We incorporate ReAlign into a two-stage MLLM training pipeline termed ReVision. In the first stage, Modality Substitution Pretraining, the ReAlign operator is used to convert large-scale long-form text into pseudo-visual representations. An adapter is trained on these representations while keeping the LLM backbone frozen, enabling the model to absorb rich world knowledge and visual semantics purely from text data, without relying on costly image-text pairs. In the second stage, Visual Instruction Tuning, real images are introduced for standard supervised learning to supplement fine-grained visual details that may be lost under purely statistical alignment, thereby refining the model's ability to follow complex instructions.

Overall, we reframe the modality gap as a structured geometric mismatch. By characterizing its first-order, second-order, and normalization-induced components, we derive a training-free alignment strategy that enables scalable use of unpaired data for multimodal learning. This provides both a geometric understanding of modality alignment and a practical path toward more cost-efficient MLLM scaling.

\section{The Isotropic Fallacy} \label{sec:motivation}

The C$^3$~\cite{zhang2024connect} framework established the dominant paradigm for addressing the modality gap, and subsequent state-of-the-art strategies have inherited its assumption, characterizing the gap simply as a superposition of a centroid shift and random alignment noise. While their centroid correction effectively rectifies the first-order bias, their treatment of the residual gap relies on a critical simplification: that the residual fluctuations are isotropic. We argue that this assumption is geometrically flawed. Multimodal contrastive representation distributions are inherently anisotropic, where information is encoded in a hierarchical spectral structure rather than a uniform sphere. As we demonstrate in Appendix \ref{app:isotropic_ssumption}, imposing an isotropic prior onto this structured manifold induces a spectral whitening effect, which dilutes the fine-grained semantic hierarchy and distorts the angular topology. This mismatch suggests that merely adjusting the mean is insufficient; the geometric shape of the noise matters. To address this, we must first rigorously characterize the true geometry of these fluctuations. In the following section, we introduce a formal decomposition framework to reveal that the modality gap is driven not by isotropic noise, but by highly structured, direction-dependent biases and residuals.
\section{Modality Gap} \label{sec:modality_gap}

Motivated by the isotropic fallacy discussed in Appendix~\ref{sec:motivation}, this section answers one core question: \textbf{what geometric shape does the modality gap have?} To this end, we train a dual-encoder model based on the InfoNCE loss and directly observe the evolution of representation geometry during training. Using a data-driven subspace construction, we explicitly separate bias and residual components. Our objective is to characterize, around a fixed reference time, the slow drift that persists late in training together with its second-moment structure. This perspective has two advantages: \fancynumber{1} bias and residual terms can be estimated separately; \fancynumber{2} all theoretical claims reduce to second-moment conditions estimable from finite-sample statistics. As a result, the shape of the modality gap is characterized by three estimable objects:\fancynumber{1} \textbf{first-order mean bias}, \fancynumber{2} \textbf{second-order residual}, and \fancynumber{3} \textbf{spherical centroid drift}. Appendix~\ref{sec:appendix_gap_causes} discusses why such a gap structurally persists in dual-encoder contrastive learning. Formal proofs of the decomposition and alignment properties are provided in Appendix~\ref{app:proofs}.

\subsection{Modality Gap Decomposition Framework} \label{subsec:decomposition_framework}

Directly analyzing the full gap vector mixes task-active representation directions with directions that carry little contrastive signal. We therefore construct a data-driven dominant subspace \(U\) from the probe covariance, so that the gap can be separated into components aligned with the main representation geometry and components lying in its orthogonal complement.

Prior work has shown that the optimization dynamics of multimodal contrastive learning can be decomposed into two subspaces: (1) an effective subspace, through which most gradients propagate, and (2) an ineffective subspace, in which gradients have only a limited effect. Following this perspective, we analyze the modality gap by decomposing the representation space into these two subspaces.

\textbf{Fixed Reference Frame Construction.} Let \(e_x(t),e_y(t)\in\mathbb{R}^d\) denote the unit-normalized embeddings of the image modality \(x\) and the text modality \(y\) for the same sample at training step \(t\). We fix a reference time \(t_0\), and compute the empirical covariance on a held-out probe set as \(\hat{\Sigma}(t_0):=\mathrm{Cov}_{\mathrm p}(e_x(t_0))+\mathrm{Cov}_{\mathrm p}(e_y(t_0))\). We perform eigendecomposition \(\hat{\Sigma}(t_0)=Q\Lambda Q^\top\), take the top \(r\) principal directions to construct the fixed reference subspace \(U:=\mathrm{span}\{q_1,\ldots,q_r\}\), and let \(V:=U^\perp\). Here, \(r\) is determined by a fixed energy threshold. We denote by \(P_U\) and \(P_V\) the orthogonal projectors onto \(U\) and \(V\), respectively, and keep them fixed throughout the analysis for all \(t\ge t_0\). Here, \(U\) is the dominant representation subspace estimated from the probe covariance, rather than a predefined semantic space.

\textbf{Bias-Residual Decomposition.} For a paired sample, define the modality gap as \(\Delta(t):=e_x(t)-e_y(t)\), and define the overall mean gap as \(m(t):=\mathbb{E}[\Delta(t)]\). Under the fixed frame \(U\oplus V\), the first-order mean gap is orthogonally decomposed as \(m(t)=\beta(t)+\gamma(t)\), where \(\beta(t):=P_Um(t)\) and \(\gamma(t):=P_Vm(t)\). \(\beta(t)\) is the mean component in the dominant representation subspace; \(\gamma(t)\) is the mean component in the orthogonal complement, called POB. Correspondingly, the zero-mean residual is decomposed as $\delta(t):=P_U(\Delta(t)-m(t))$, and $\zeta(t):=P_V(\Delta(t)-m(t))$. Therefore, the fixed-reference decomposition of the modality gap is:
\begin{equation}
    \Delta(t)=\beta(t)+\gamma(t)+\delta(t)+\zeta(t).
\end{equation}
Here, \(\beta(t)\) and \(\gamma(t)\) describe the first-order mean shape, while \(\delta(t)\) and \(\zeta(t)\) describe the second-order residual shape.

Along the temporal dimension, we focus on a short late-training window around the reference time \(t_0\), denoted by \(t\in\mathcal{T}:=\{t_0,\ldots,t_0+\tau-1\}\). Within this window, the fixed reference subspaces \(U,V\) and the projectors \(P_U,P_V\) remain unchanged, so all components \(\beta(t),\gamma(t),\delta(t),\zeta(t)\) are compared in the same reference frame. To characterize the rotation of the dominant representation subspace as training progresses, we also re-estimate an instantaneous probe subspace \(U_t\) at each logging step using the same covariance construction, and measure its geometric deviation from the fixed reference subspace by the largest principal angle \(\theta(U_t,U)\).


\begin{figure*}[htbp]
  \centering
  \includegraphics[width=1.\linewidth]{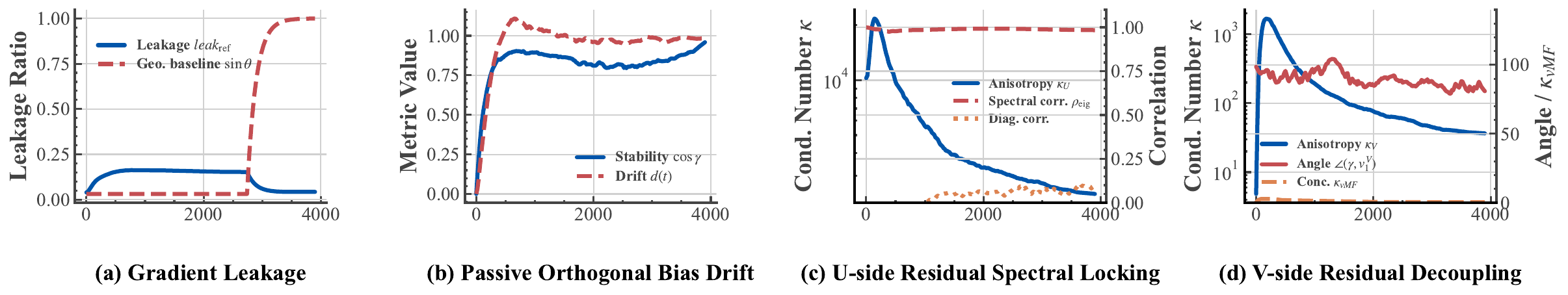}
  \caption{Geometric statistics of the modality gap, showing gradient leakage, passive bias drift, and anisotropic residual structures in the fixed \(U\oplus V\) reference frame.}
  \label{fig:bias_res_fig}
\end{figure*}

\subsection{Bias Terms: First-order Bias} \label{subsec:bias}

This section characterizes the shape of the first-order bias in the modality gap. Under the fixed reference frame \(U \oplus V\), the mean modality gap can be decomposed as \(m(t)=\beta(t)+\gamma(t)\), where \(\beta(t)=P_Um(t)\) and \(\gamma(t)=P_Vm(t)\). Therefore, the first-order bias consists of two components: the principal mean bias \(\beta(t)\) in the dominant representation subspace \(U\), and the passive orthogonal bias \(\gamma(t)\) in the orthogonal complement \(V\).

\textbf{PMB: Mean Bias in the Dominant Subspace.} \(\beta(t)\) is the mean displacement of the modality gap inside the dominant representation subspace \(U\), which we refer to as the \textbf{principal mean bias}, or PMB. Geometrically, it corresponds to the projection of the anchor mismatch between the two modality representation distributions onto \(U\). Therefore, \(\beta(t)\) is not residual noise, but a deterministic centroid offset along task-active representation directions. The main statistical role of PMB is to center the \(U\)-side residual. Since \(\delta(t)=P_U(\Delta(t)-m(t))\), removing \(\beta(t)=P_Um(t)\) makes \(\delta(t)\) a zero-mean residual. Otherwise, the estimate of \(\Sigma_U\) would contain both the first-order mean displacement and the second-order residual fluctuation, and thus could not serve as a pure statistic of residual shape. Therefore, PMB is both the anchor displacement inside the dominant subspace and the centering term required for the covariance analysis of the \(U\)-side residual.

\textbf{POB: Passive Bias in the Orthogonal Complement.}
\(\gamma(t)\) is the mean displacement of the modality gap inside the orthogonal complement \(V\), which we refer to as the \textbf{passive orthogonal bias}, or POB. Unlike \(\beta(t)\), which lies in the dominant representation subspace, \(\gamma(t)\) lies outside the subspace where most late-stage optimization signals concentrate. This makes it a passive bias: it is not directly suppressed by a strong gradient component along \(V\), but instead evolves slowly under weak residual updates. To quantify this effect, we measure the reference leakage ratio \(\operatorname{leak}_{\mathrm{ref}}(t)\), defined as the fraction of embedding-level gradient RMS energy projected onto \(V\). Fig.~\ref{fig:bias_res_fig}(a) shows that this leakage remains low in late training. This is consistent with the structure of InfoNCE: embedding-level gradients are constrained by the span of the current contrastive set, and late-stage embeddings concentrate near the instantaneous dominant subspace \(U_t\). Since the fixed reference subspace \(U\) may rotate relative to \(U_t\), the observed leakage along \(V\) is compared with the geometric baseline \(\sin\theta(U_t,U)\). Fig.~\ref{fig:bias_res_fig}(a) shows that this leakage remains low in late training, indicating that only a small fraction of the optimization signal directly acts along \(V\). We further compare it with the geometric baseline \(\sin\theta(U_t,U)\), which measures the \(V\)-projection induced by the rotation of the instantaneous dominant subspace relative to the fixed reference frame. This comparison suggests that late-stage \(V\)-direction optimization pressure is weak, explaining why \(\gamma(t)\) is not rapidly eliminated. Appendix~\ref{sec:appendix_weak_coupling} further analyzes possible weak coupling between \(U\)- and \(V\)-side residuals. Consequently, \(\gamma(t)\) is not rapidly eliminated. Fig.~\ref{fig:bias_res_fig}(b) shows that it maintains high adjacent-step cosine stability while undergoing slow cumulative drift. We characterize this behavior by the relative drift magnitude \(\operatorname{drift}(t):=\|\gamma(t)-\gamma(t_0)\|/\max\{\|\gamma(t_0)\|,\varepsilon\}\) and the adjacent-step cosine stability \(\cos(\gamma(t),\gamma(t-1))\). In late training, this can be approximated as \(\gamma(t+1)=\gamma(t)+\eta(t)\), where \(\|\eta(t)\|\) is small. Thus, the typical shape of POB is not rapid convergence to zero, but slow drift along an approximately stable direction.


\begin{tcolorbox}[colframe=mitred, opacityback=0.9, arc=4pt, left=4pt, right=4pt, top=2pt, bottom=0pt,
title=Conclusion 1. Bias Shape]
The first-order bias is a decomposable anchor displacement, \(m(t)=\beta(t)+\gamma(t)\). Here, \(\beta(t)\) is the anchor displacement inside the dominant representation subspace, serving to center the residual along task-active directions, while \(\gamma(t)\) is the passive mean bias in the orthogonal complement, whose direction remains stable and drifts slowly because late-stage gradients along \(V\) are weak. Therefore, the first-order structure of the modality gap should be understood as a translation-correctable anchor mismatch at the mean level.
\end{tcolorbox}

\subsection{Residual Terms: Anisotropic Second-order Residual} \label{subsec:res}

After the first-order mean is decomposed and centered, can the remaining residual be regarded as isotropic white noise? We show that the answer is no: the residual itself has a stable second-order shape. For the late-training window \(s=0,\ldots,\tau-1\), we compute the residual covariances at each step as \(\Sigma_U^{(s)}:=\operatorname{Cov}_P(\delta(t_0+s))\) and \(\Sigma_V^{(s)}:=\operatorname{Cov}_P(\zeta(t_0+s))\), and then average them over time as \(\bar{\Sigma}_U:=\frac{1}{\tau}\sum_{s=0}^{\tau-1}\Sigma_U^{(s)}\) and \(\bar{\Sigma}_V:=\frac{1}{\tau}\sum_{s=0}^{\tau-1}\Sigma_V^{(s)}\). This step-wise covariance estimation avoids mixing the slow drift of POB into the residual covariance. We use the trace to measure the residual energy scale, and the effective condition number \(\kappa_{\mathrm{eff}}\) to measure the spectral stretching of the dominant residual components.

\textbf{\(U\)-side Residual.}
The residual \(\delta(t)\) lies in the dominant representation subspace \(U\). Fig.~\ref{fig:bias_res_fig}(c) shows that \(\kappa_{\mathrm{eff}}(\bar{\Sigma}_U)\) remains high throughout training, indicating that the \(U\)-side residual does not diffuse uniformly. Instead, its energy concentrates along a small number of principal directions. This suggests a stable spectral hierarchy inside the dominant representation subspace, where different directions carry different amounts of residual energy. To compare this residual structure with the optimization signal, let \(G_U(t)\) denote the embedding-level gradient covariance restricted to \(U\), and define \(\rho_{\mathrm{eig}}(t):=\operatorname{corr}[\lambda(\Sigma_U(t)),\lambda(G_U(t))]\), where \(\lambda(\cdot)\) denotes the eigenvalue vector arranged in descending order. Fig.~\ref{fig:bias_res_fig}(c) shows that \(\rho_{\mathrm{eig}}(t)\) rises rapidly in early training and remains stable in late training. This indicates a stable correspondence between the spectral energy profile of the \(U\)-side residual covariance and that of the gradient covariance. Therefore, the \(U\)-side residual is a structured fluctuation with a stable spectral hierarchy, rather than an isotropic perturbation.

\textbf{\(V\)-side Residual.}
The residual \(\zeta(t)\) lies in the orthogonal complement \(V\). Fig.~\ref{fig:bias_res_fig}(d) shows that \(\kappa_{\mathrm{eff}}(\bar{\Sigma}_V)\) remains higher than the isotropic baseline, indicating that the \(V\)-side residual also forms an elongated ellipsoidal shape. Although its total energy \(\operatorname{tr}(\bar{\Sigma}_V)\) is lower than that of the \(U\)-side residual, it is still directionally structured. Figure 1(d) further shows that \(\gamma(t)\) is persistently close to orthogonal to the first principal direction of \(\bar{\Sigma}_V\). If POB were merely the leading direction of \(V\)-side residual spread, these two directions would align. Their near-orthogonality shows that the first-order mean bias and the second-order residual anisotropy are geometrically decoupled in \(V\): the former is a stable mean displacement, while the latter is a residual ellipsoid with its own spectral structure.

\begin{tcolorbox}[colframe=mitred, opacityback=0.9, arc=4pt, left=4pt, right=4pt, top=2pt, bottom=0pt,
title=Conclusion 2. Residual Shape]
After mean alignment, the residual is an anisotropic distributional mismatch. On the \(U\)-side, the residual energy concentrates along a few dominant spectral directions and forms a stable spectral hierarchy that is aligned with the spectral profile of the optimization signal. On the \(V\)-side, the residual has lower total energy but still forms an elongated ellipsoid, and this ellipsoid is geometrically decoupled from the POB mean direction \(\gamma(t)\). Therefore, the second-order structure of the modality gap is a stable anisotropic residual shape.
\end{tcolorbox}


\subsection{Phantom Drift: Normalization-induced Secondary Centroid Drift} \label{subsec:phantom_drift}

If the first-order mean is already aligned in Euclidean space, can we guarantee that the centroid of the normalized embeddings is also aligned? The answer is no: spherical normalization couples the first-order anchor with the second-order residual shape, producing Phantom Drift.

Let the source-modality representation after linear mean alignment be \(z=\mu+\epsilon\), where \(\mathbb{E}[\epsilon]=0\) and \(\mathrm{Cov}(\epsilon)=\Sigma\). Here, \(\mu\) is the aligned target anchor, and \(\epsilon\) is the zero-mean residual. In Euclidean space, \(\mathbb{E}[z]=\mu\), so the first-order mean has already been aligned. However, the actual embedding is projected onto the unit sphere by \(\pi(z):=z/\|z\|\). The centroid after normalization is \(\mu^\pi:=\mathbb{E}[\pi(z)]=\mathbb{E}[(\mu+\epsilon)/\|\mu+\epsilon\|]\). Since \(\pi(\cdot)\) is nonlinear, in general \(\mu^\pi \neq \pi(\mu)\). We define the normalization-induced secondary centroid drift as \(\Delta^\pi:=\mu^\pi-\pi(\mu)=\mathbb{E}[(\mu+\epsilon)/\|\mu+\epsilon\|]-\mu/\|\mu\|\). This drift depends on both the anchor direction \(\mu\) and the residual covariance \(\Sigma\), so an anisotropic residual can shift the spherical centroid even after Euclidean mean alignment. This is \textbf{Phantom Drift}.

\begin{tcolorbox}[colframe=mitred, opacityback=0.9, arc=4pt, left=4pt, right=4pt, top=2pt, bottom=0pt,
title=Conclusion 3: Phantom Drift]
Mean alignment in Euclidean space does not guarantee centroid alignment on the unit hypersphere. After normalization, the anchor direction couples with the residual covariance structure, causing a secondary shift of the spherical centroid of the representation distribution.
\end{tcolorbox}
\section{ReAlign: Training-Free Modality Alignment} \label{sec:realign}

Sec.~\ref{sec:modality_gap} shows that the modality gap contains three geometric components: first-order anchor displacement, anisotropic second-order residual mismatch, and spherical-normalization-induced centroid drift. Based on this decomposition, we propose \textbf{ReAlign}, as shown in Fig.~\ref{fig:realign}, a training-free modality alignment operator that estimates low-order statistics from large-scale unpaired data and performs three closed-form calibrations: \fancynumber{1} \textbf{Anchor Alignment} corrects the first-order anchor displacement, \fancynumber{2} \textbf{Trace Alignment} matches the global residual energy scale while preserving the learned anisotropic spectral structure, and \fancynumber{3} \textbf{Centroid Alignment} corrects the Phantom Drift induced by spherical normalization.

\begin{figure*}[t]
    \centering
    \definecolor{sybcblue}{HTML}{004488}  
    \definecolor{brickred}{HTML}{BB5566}  
    \definecolor{softgray}{HTML}{666666}  

    \resizebox{1.0\textwidth}{!}{%
    \begin{tikzpicture}[
        >={Latex[length=4mm, width=3mm]}, 
        font=\sffamily,
        target/.style={ellipse, draw=sybcblue, very thick, fill=sybcblue!10, minimum width=2.8cm, minimum height=1.4cm, rotate=0, align=center},
        source/.style={ellipse, draw=brickred, very thick, dashed, fill=brickred!10, minimum width=1.5cm, minimum height=1.2cm, rotate=0, align=center},
        ghost/.style={ellipse, draw=softgray!60, very thick, dashed, fill=none, minimum width=1.5cm, minimum height=1.2cm, rotate=0},
        axis/.style={->, draw=softgray, very thick},
        arrow label/.style={midway, above, softgray, font=\footnotesize},
        math formula/.style={font=\small\boldmath\bfseries, text=softgray, inner sep=1pt, align=center} 
    ]
    
    \begin{scope}[local bounding box=panelA]
        \draw[axis] (-2,0) -- (2,0) node[right, text=softgray] {$U$};
        \draw[axis] (0,-1.4) -- (0,1.8) node[above, text=softgray] {$V$};
        
        \node[target] (tgtA) at (0,0) {};
        \node[below right, sybcblue] at (tgtA.south east) {Target ($\mu_x$)};
        
        \node[source, minimum width=1.6cm, minimum height=0.9cm] (srcA) at (-1.3, 0.9) {};
        \node[above left, brickred, xshift=1.0cm, yshift=0.1cm] at (srcA.north west) {Source ($\mu_y$)};
        
        
        \node[below=2.6cm of tgtA, text=black] {(a) \textbf{Original State}};
    \end{scope}
    
    \draw[->, line width=1.2mm, softgray!30] ($(panelA.east)+(0.2,0)$) -- ($(panelA.east)+(1.2,0)$)
        node[arrow label] {Step 1};

    \begin{scope}[xshift=6.5cm, local bounding box=panelB]
        \draw[axis] (-2,0) -- (2,0) node[right, text=softgray] {$U$};
        \draw[axis] (0,-1.4) -- (0,1.8) node[above, text=softgray] {$V$};
        
        \node[target] (tgtB) at (0,0) {};
        
        \node[ellipse, draw=brickred!60, dashed, very thick, minimum width=1.6cm, minimum height=0.9cm] (ghostB) at (-1.3, 0.9) {};
        
        \node[source, minimum width=1.6cm, minimum height=0.9cm] (srcB) at (0,0) {};
        
        
        \draw[->, brickred, very thick] (ghostB.center) -- (srcB.center) 
             node[midway, above right, font=\tiny, sloped, yshift=-2pt] {Anchor Aligned};
        
        \node[math formula] at (0, -2.1) {$\dot{e}_y = (e_y - \mu_y) + \mu_x$};
        
        \node[below=2.6cm of tgtB, align=center, text=black] {(b) \textbf{Step1: Anchor Alignment}};
    \end{scope}

    \draw[->, line width=1.2mm, softgray!30] ($(panelB.east)+(0.2,0)$) -- ($(panelB.east)+(1.2,0)$)
        node[arrow label] {Step 2};

    \begin{scope}[xshift=13cm, local bounding box=panelC]
        \draw[axis] (-2,0) -- (2,0) node[right, text=softgray] {$U$};
        \draw[axis] (0,-1.4) -- (0,1.8) node[above, text=softgray] {$V$};
        
        \draw[softgray, thick, dashed, opacity=0.5] (0,0) circle (1.4cm);

        \node[target] (tgtC) at (0,0) {};
        \node[ghost, minimum width=1.6cm, minimum height=0.9cm] (ghostC) at (0,0) {};
        
        \node[source, minimum width=2.8cm, minimum height=1.4cm] (srcC) at (0.4, 0.3) {};
        
        \draw[->, brickred, thick] (ghostC.east) -- (srcC.east);
        \draw[->, brickred, thick] (ghostC.north) -- (srcC.north);
        \node[right, brickred, font=\tiny] at (1.5, 0.8) {Scale $s$};
        
        \node[brickred, font=\tiny, inner sep=0.5pt] at (0.4, 0.3) {Centroid $\mu'$};
        \draw[->, brickred, dotted, very thick] (0,0) -- (0.4, 0.3) node[midway, below right, font=\tiny] {Drift};

        \node[math formula] at (0, -2.1) {$\tilde{e}_y = \mu_x + s(e_y - \mu_y)$};

        \node[below=2.6cm of tgtC, align=center, text=black] {(c) \textbf{Step2: Trace Alignment}};
    \end{scope}

    \draw[->, line width=1.2mm, softgray!30] ($(panelC.east)+(0.2,0)$) -- ($(panelC.east)+(1.2,0)$)
        node[arrow label] {Step 3};

    \begin{scope}[xshift=19.5cm, local bounding box=panelD]
        \draw[axis] (-2,0) -- (2,0) node[right, text=softgray] {$U$};
        \draw[axis] (0,-1.4) -- (0,1.8) node[above, text=softgray] {$V$};
        
        \node[target] (tgtD) at (0,0) {};
        
        \node[ghost, minimum width=2.8cm, minimum height=1.4cm] (preD) at (0.4, 0.3) {};
        \node[softgray, font=\tiny] at (1.0, 1.2) {Phantom Drift};
        
        \node[source, minimum width=2.8cm, minimum height=1.4cm] (srcD) at (0,0) {};
        
        \draw[->, brickred, very thick] (preD.center) -- (srcD.center)
            node[midway, above left, font=\tiny, align=right] {Correction};
        
        \node[below right, sybcblue, font=\tiny] at (tgtD.south east) {Gap $\approx 0$};
        
        \node[math formula] at (0, -2.1) {$e''_y = e'_y - \mu' + \mu_x$};

        \node[below=2.6cm of tgtD, align=center, text=black] {(d) \textbf{Step3: Centroid Align}};
    \end{scope}

    \draw[->, line width=1.2mm, softgray!30] ($(panelD.east)+(0.2,0)$) -- ($(panelD.east)+(1.2,0)$)
        node[arrow label] {Final};

    \begin{scope}[xshift=26cm, local bounding box=panelE]
        \draw[axis] (-2,0) -- (2,0) node[right, text=softgray] {$U$};
        \draw[axis] (0,-1.4) -- (0,1.8) node[above, text=softgray] {$V$};
        
        \draw[softgray, very thick, dashed] (0,0) circle (1.6cm);
        \node[softgray, font=\tiny, above right] at (1.1, 1.1) {Unit Sphere};
        
        \node[target] (tgtE) at (0,0) {};
        \node[source, minimum width=2.8cm, minimum height=1.4cm] (srcE) at (0,0) {};
        
        \node[font=\tiny, align=center, inner sep=1pt, text=black] at (0,0) {\textbf{Aligned} ($\hat{e}_y$)};
        \node[below left, brickred, font=\tiny] at (srcE.south west) {Norm Match};
        \node[below right, sybcblue, font=\tiny] at (tgtE.south east) {Center Match};
        
        \node[math formula] at (0, -2.1) {$\hat{e}_y = e''_y / \|e''_y\|$};

        \node[below=2.6cm of tgtE, align=center, text=black] {(e) \textbf{Final State}};
    \end{scope}

    \end{tikzpicture}
    }
    \caption{The ReAlign pipeline. ReAlign sequentially performs Anchor Alignment, Trace Alignment, and Centroid Alignment to align the source-modality distribution toward the target-modality distribution while preserving its geometric structure on the unit hypersphere.}

    
    \label{fig:realign}
    \vskip -0.1in
\end{figure*}


\subsection{Step 1: Anchor Alignment} \label{subsec:anchor}

\noindent\textbf{Motivation.}
Sec.~\ref{subsec:bias} shows that the first-order structure of the modality gap appears as an anchor displacement between the two modality representation distributions. This displacement is not a random perturbation, but an estimable and translation-correctable global centroid mismatch at the mean level. Therefore, the first step of ReAlign is to translate the source-modality anchor to the target-modality anchor.

\noindent\textbf{Operation.}
Let \(e_y,e_x\in\mathbb{R}^d\) denote the unit-normalized embeddings of the source modality \(y\) and the target modality \(x\), respectively. Their population means are \(\mu_y:=\mathbb{E}[e_y]\) and \(\mu_x:=\mathbb{E}[e_x]\). To map the source modality \(y\) to the reference position of the target modality \(x\), we first center the source-modality representation and then translate it to the target anchor:
\[
\dot{e}_y=(e_y-\mu_y)+\mu_x.
\]
This operation satisfies \(\mathbb{E}[\dot{e}_y]=\mu_x\). Therefore, Anchor Alignment removes the first-order mean mismatch between the source and target modalities. Anchor Alignment only corrects the first-order anchor position and does not destroy the second-order spectral structure of the source representation.

\subsection{Step 2: Trace Alignment} \label{subsec:trace}

\noindent\textbf{Motivation.}
Sec.~\ref{subsec:res} shows that the residual after mean alignment is not isotropic noise, but an anisotropic residual shape with a stable spectral hierarchy. Therefore, full whitening or isotropic noise injection may remove the learned anisotropic spectral structure that carries task-relevant information.

\noindent\textbf{Observation.}
Empirically, after contrastive pretraining, the two modalities exhibit compatible trace-normalized spectral profiles. Let the centered covariances be \(\Sigma_y:=\mathrm{Cov}(e_y-\mu_y)\) and \(\Sigma_x:=\mathrm{Cov}(e_x-\mu_x)\). Although \(\Sigma_y\) and \(\Sigma_x\) may differ in their overall residual energy scale, their trace-normalized spectra are close:
\[
\frac{\lambda(\Sigma_y)}{\mathrm{tr}(\Sigma_y)}
\approx
\frac{\lambda(\Sigma_x)}{\mathrm{tr}(\Sigma_x)}.
\]
This suggests that the remaining second-order calibration should primarily match the global residual energy scale, rather than whiten or reshape the full covariance. Appendix~\ref{sec:training_free_substitution} shows why a stronger blockwise covariance alignment can be less effective.

\noindent\textbf{Operation.}
We define the residual energy of each modality as the trace of its centered distribution: \(T_y:=\mathbb{E}[\|e_y-\mu_y\|^2]=\mathrm{tr}(\Sigma_y)\) and \(T_x:=\mathbb{E}[\|e_x-\mu_x\|^2]=\mathrm{tr}(\Sigma_x)\). We compute the trace-matching factor \(s=\sqrt{T_x/(T_y+\varepsilon)}\), where \(\varepsilon\) is a small constant for numerical stability. Combining Anchor Alignment and Trace Alignment gives the affine-transformed source-modality representation:
\[
\tilde{e}_y=\mu_x+s(e_y-\mu_y).
\]
This transformation satisfies \(\mathbb{E}[\tilde{e}_y]=\mu_x\), and its residual energy is \(\mathbb{E}[\|\tilde{e}_y-\mu_x\|^2]=s^2T_y\approx T_x\). This scalar scaling changes the residual energy scale without changing the eigenvectors or the trace-normalized spectrum of the source covariance. After this, we perform the first spherical projection:
\[
e'_y=\frac{\tilde{e}_y}{\|\tilde{e}_y\|}.
\]

\subsection{Step 3: Centroid Alignment} \label{subsec:centroid}

\noindent\textbf{Motivation.}
Sec.~\ref{subsec:phantom_drift} shows that even after Anchor Alignment and Trace Alignment calibrate the mean and residual energy in Euclidean space, spherical normalization still induces a new centroid shift. Therefore, the final step of ReAlign explicitly corrects this Phantom Drift.

\noindent\textbf{Operation.}
Let the centroid after the first spherical projection be \(\mu'_y:=\mathbb{E}[e'_y]\). Here, \(\mu_x\) is used as the empirical spherical target centroid. The corresponding Phantom Drift is \(\Delta^\pi_y:=\mu'_y-\mu_x\). To correct this spherical centroid shift, we translate the projected source-modality representation back to the target anchor:
\[
e''_y=e'_y-\mu'_y+\mu_x.
\]
Finally, we normalize again to return the representation to the unit hypersphere:
\[
\hat{e}_y=\frac{e''_y}{\|e''_y\|}, \qquad \hat{e}_x=e_x.
\]
In practice, this second projection keeps the embeddings on the unit sphere while leaving only a small residual centroid error. We study the sample efficiency, numerical stability, and domain sensitivity of ReAlign in Appendix~\ref{app:robustness_analysis}.

\section{ReVision: A Scalable MLLM Training Paradigm}

We introduce \textbf{ReVision}, a two-stage training paradigm for MLLMs. Building on the geometry-preserving ReAlign strategy, ReVision synthesizes pseudo-modality features from unpaired text corpora. This design enables low-cost semantic injection, allowing the model to absorb extensive world knowledge during pretraining without relying on expensive, high-quality paired data.

\paragraph{Stage 1: Modality Substitution Pretraining}

Let $E_{\text{img}}$ and $E_{\text{text}}$ denote the frozen image and text encoders. Freed from the reliance on paired visual-text data, we leverage $E_{\text{text}}$ to encode large-scale unpaired text corpora. Appendix~\ref{app:long_caption_paradox} analyzes why longer captions do not necessarily provide better modality substitution signals. Unlike traditional multimodal datasets limited by data scarcity, this strategy allows training to scale to raw text resources. As a result, the model receives dense semantic supervision and can absorb broader world knowledge during pretraining than standard approaches.

We define a training-free modality substitution operator $S_{y \to x}$ based on Section~\ref{sec:realign}. Given a text sample $y$, this operator maps its embedding into the image space distribution $\tilde{e}_x = S_{y \to x}(E_{\text{text}}(y))$. Here, $\tilde{e}_x$ serves as a pseudo-visual embedding. Owing to the geometry-preserving nature of ReAlign, it preserves the rich semantics of $y$ while strictly adhering to the anisotropic geometric statistics of real images. We train the adapter $\phi$ (with LLM $\theta$ frozen) to reconstruct the text conditioned on the pseudo-visual embedding $\tilde{e}_x$:
\begin{equation}
\mathcal{L}_{\text{pre}}(\phi) = -\sum_{t=1}^{T} \log p_\theta (y_t \mid y_{<t}, T_\phi(\tilde{e}_x)).
\end{equation}
\paragraph{Stage 2: Visual Instruction-Tuning}
While Stage 1 establishes geometric compatibility, Stage 2 focuses on enhancing capabilities in more challenging scenarios. In this stage, real image embeddings $e_x = E_{\text{img}}(x)$ are introduced. Real images provide fine-grained visual details that may be abstracted away under purely statistical alignment, and are essential for handling complex instructions and intricate reasoning tasks. We fine-tune the model using standard supervised instruction tuning:
\begin{equation}
\mathcal{L}_{\text{sft}}(\theta, \phi) = -\sum_{t=1}^{L} \log p_\theta (r_t \mid r_{<t}, I, T_\phi(e_x)),
\end{equation}
where $(I, r)$ denotes an image-instruction pair.

\paragraph{Inference}

During inference, the model directly takes real images as input. This inherent compatibility stems from the asymmetric alignment strategy. By aligning the text representation distribution to the image representation distribution during pretraining, the model supports single-image inference without relying on statistics from multiple images for calibration, incurring no additional computational overhead.
\section{Experiments}
\label{sec:experiments}

In this section, we systematically investigate the effectiveness of ReVision under the proposed geometry-preserving framework. Our experiments are designed to address three core research questions (RQs). To ensure a fair comparison, all methods employ the same model architecture. We use LLM2CLIP-Openai-L-14-336~\cite{huang2024llm2clip} as the encoder and Llama-3-8B-Instruct~\cite{dubey2024llama} as the LLM backbone. For ReVision training, we use bunny-pretrain~\cite{he2024efficient} during the Modality Substitution Pretraining stage and InternVL-Chat-V1-2-SFT~\cite{chen2024internvl} for the Visual Instruction Tuning stage. Detailed training, evaluation, and cost settings are provided in Appendix~\ref{app:setting}.

\begin{wrapfigure}{r}{0.6\columnwidth}
\vspace{-8pt}
\centering
\includegraphics[width=\linewidth]{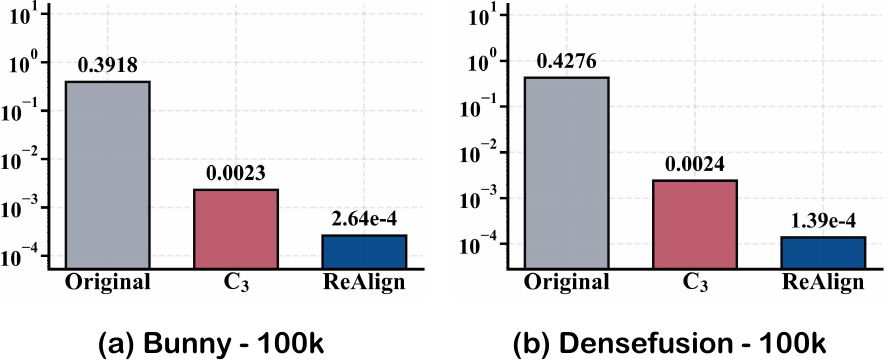}
\caption{We measure the modality gap between aligned centroids on Bunny and DenseFusion. While the baseline C3 stagnates at a geometric bottleneck ($\approx 0.0023$) due to isotropic assumptions, ReAlign reduces the gap to the $10^{-4}$ scale by effectively modeling anisotropic covariance.}
\vskip -0.1in
\label{fig:gap_analysis}
\end{wrapfigure}

\textbf{\textcolor{wine}{RQ1.} Does preserving anisotropy reduce the modality gap more effectively than isotropic corruption?} We first verify the geometric alignment quality by quantifying the Euclidean distance between the centroids of the aligned modalities. We performed alignment on 100k samples from the Bunny-pretrain~\cite{he2024efficient} and DenseFusion~\cite{li2024densefusion} datasets. As illustrated in Figure~\ref{fig:gap_analysis}, the original representations exhibit a substantial gap ($\approx 0.4$). While C$^3$~\cite{zhang2024connect} reduces the gap significantly via centroid subtraction, its reliance on isotropic noise limits the precision of the fit. The residual misalignment suggests that spherical noise cannot perfectly cover the anisotropic visual manifold. By explicitly modeling the covariance structure and correcting for manifold drift, ReAlign reduces the gap by orders of magnitude. On Bunny, the gap drops to $\mathbf{2.64 \times 10^{-4}}$, and on DenseFusion to $\mathbf{1.39 \times 10^{-4}}$. Notably, we observe that while the initial gaps vary across datasets ($0.3918$ for Bunny, $0.4276$ for DenseFusion), C$^3$ stagnates at a similar level ($\approx 0.0023$) for both. This performance plateau indicates a geometric bottleneck: isotropic noise lacks the flexibility to adapt to the representation variations of different distributions. In contrast, ReAlign breaks this bottleneck by adaptively matching the covariance trace of each dataset, demonstrating superior adaptability to the intrinsic geometric structure of diverse data sources. Additional spectral, angular, and visualization diagnostics are reported in Appendix~\ref{app:isotropic_ssumption}.

\textbf{\textcolor{wine}{RQ2.} How effective is ReAlign in large-scale MLLM training scenarios?} To ensure a fair comparison, all baselines adopt the same two-stage paradigm as ReVision: Stage 1 utilizes 1M text samples from Bunny-pretrain, followed by Stage 2 using the InternVL-Chat-V1-2-SFT. We compare ReVision against: \fancynumber{1} \textbf{Blind:} We directly evaluate Qwen3-235B-A22B-Instruct~\cite{yang2025qwen3} on text-only questions without providing image inputs. This demonstrates that without visual perception, even the most powerful LLMs falter on multimodal tasks. \fancynumber{2} \textbf{No Align:} Utilizing raw text representations. \fancynumber{3} \textbf{C$^3$ Align:} Adopting the C$^3$ strategy, which performs centroid alignment and injects isotropic Gaussian noise into text representations. Results are shown in the Table~\ref{tab1}. ReVision achieves the highest average score of 50.16, significantly outperforming C$^3$ (48.06). In reasoning-intensive benchmarks, ReVision surpasses C$^3$. This supports our point that C$^3$'s isotropic noise induces a whitening effect that erodes the fine-grained semantic hierarchy essential for complex reasoning. ReAlign preserves the spectral decay, maintaining the structural richness of the features. ReVision demonstrates a clear advantage in hallucination metrics (CRPE:$81.78$, HallBench:$46.58$). We attribute this to the correction of Phantom Drift. By ensuring the centroid is aligned on the hypersphere, ReVision prevents the projection layer from overfitting to spurious directional biases, enabling more faithful visual grounding. Additional MLLM ablations and scaling results are provided in Appendix~\ref{app:more_mllm_experiments}.

\begin{table*}[t]
\centering
\setlength{\tabcolsep}{2pt}
\scriptsize
\caption{Performance comparison of different geometric alignment strategies.}
\vspace{-5pt}
\label{tab1}
\renewcommand{\arraystretch}{1.2}
\resizebox{0.95\textwidth}{!}{%
\begin{tabular}{lcccccccccccc}
\toprule
\multirow{2}{*}{\textbf{Method}}
& \multicolumn{4}{c}{\textbf{General}}
& \multicolumn{4}{c}{\textbf{Reasoning}}
& \multicolumn{3}{c}{\textbf{Hallucination}}
& \multirow{2}{*}{\textbf{Avg. $\uparrow$}} \\
\cmidrule(lr){2-5} \cmidrule(lr){6-9} \cmidrule(lr){10-12}
& \texttt{MME} & \texttt{MMStar} & \texttt{SQA} & \texttt{RealWorldQA}
& \texttt{MMMU} & \texttt{MMMU-P} & \texttt{VisuLogic} & \texttt{LogicVista}
& \texttt{CRPE} & \texttt{POPE} & \texttt{HallBench} & \\
\midrule
Blind & 3.37 & 8.80 & 6.17 & 5.36 & 19.60 & 12.44 & 0.30 & 1.56 & 12.90 & 0.60 & 15.25 & 7.85 \\
\hdashline
No Align & 73.63 & 35.73 & 75.23 & 43.53 & 28.82 & 25.38 & 24.40 & 21.03 & 80.82 & 71.59 & 42.38 & 47.50 \\
C$^3$ Align~\cite{zhang2024connect} & 76.16 & 34.60 & 75.52 & 43.14 & 30.69 & 27.20 & 25.50 & 19.91 & 79.99 & 72.43 & 43.53 & 48.06 \\
\hdashline
\rowcolor{graybg} \textbf{ReVision (ours)} & \textbf{79.65} & \textbf{36.13} & \textbf{76.71} & \textbf{47.97} & \textbf{31.51} & \textbf{28.39} & \textbf{27.70} & \textbf{22.82} & \textbf{81.78} & \textbf{72.53} & \textbf{46.58} & \textbf{50.16} \\
\bottomrule
\end{tabular}
}
\end{table*}

\begin{table*}[t]
\centering
\setlength{\tabcolsep}{1.5pt}
\setlength{\tabcolsep}{2pt}
\scriptsize
\caption{Cost-performance comparison between paired image-text pretraining and text-only ReVision scaling.} 
\vspace{-5pt}
\label{tab2}
\renewcommand{\arraystretch}{1.2}
\resizebox{0.95\textwidth}{!}{%
\begin{tabular}{lcccccccccccccc}
\toprule
\multirow{2}{*}{\textbf{Method}}
& \multirow{2}{*}{\textbf{Text-only}}
& \multicolumn{4}{c}{\textbf{General}}
& \multicolumn{4}{c}{\textbf{Reasoning}}
& \multicolumn{3}{c}{\textbf{Hallucination}}
& \multirow{2}{*}{\textbf{Avg. $\uparrow$}}
& \multirow{2}{*}{\textbf{Cost $\downarrow$}} \\
\cmidrule(lr){3-6} \cmidrule(lr){7-10} \cmidrule(lr){11-13}
&& \texttt{MME} & \texttt{MMStar} & \texttt{SQA} & \texttt{RealWorldQA}
& \texttt{MMMU} & \texttt{MMMU-P} & \texttt{VisuLogic} & \texttt{LogicVista}
& \texttt{CRPE} & \texttt{POPE} & \texttt{HallBench} & & \\
\midrule
Paired image-text & \xmark & 73.59 & 35.40 & 76.01 & 44.18 & 34.80 & 27.70 & 25.90 & 23.27 & \textbf{80.87} & 69.13 & 47.11 & 48.91 & 1.00 \\
\hdashline
Unicorn~\cite{yu2025unicorn} & \cmark & 60.24 & 35.13 & 68.81 & 42.35 & \textbf{36.87} & \textbf{34.05} & \textbf{26.80} & \textbf{29.53} & 42.32 & 64.21 & 43.01 & 43.94 & 3.98~\textsubscript{\textcolor{crimson}{(+298\%)}} \\
\textbf{ReVision-1M (ours)} & \cmark & 72.20 & 34.33 & 75.84 & 43.72 & 30.22 & 27.64 & 25.70 & 21.03 & 79.59 & 71.93 & 46.37 & 48.05 &  \textbf{0.37}~\textsubscript{\textcolor{forestgreen}{(-63\%)}}\\
\hdashline
\rowcolor{graybg} \textbf{ReVision-2M (ours)} & \cmark
& \textbf{74.94} & \textbf{36.40} & \textbf{76.35} & \textbf{45.23}
& 33.49 & 29.59 & \textbf{26.80} & 24.38
& 80.14 & \textbf{72.18} & \textbf{48.26} & \textbf{49.75} & 0.74~\textsubscript{\textcolor{forestgreen}{(-26\%)}} \\
\bottomrule
\end{tabular}%
}
\vskip -0.1in
\end{table*}

\textbf{\textcolor{wine}{RQ3.} Can scaling up low-cost text-only pretraining surpass the performance of expensive paired image-text training?} Finally, we investigate whether ReVision can challenge the traditional paradigm that relies on expensive paired image-text data. To ensure fairness, we sampled a subset of 417k examples from the SFT dataset for this comparison, aligning exactly with the SFT data scale used in Unicorn~\cite{yu2025unicorn}. We define the Cost metric based on the expense of data synthesis using GPT5 APIs, normalizing the cost of 1M image-text pairs to 1.0. We compare four pre-training settings in the Table~\ref{tab2}: \fancynumber{1} \textbf{Paired image-text:} 1M ground-truth image-text pairs (upper bound). \fancynumber{2} \textbf{Unicorn:} 1M unpaired text samples, text-only method using mean shift. \fancynumber{3} \textbf{ReVision-1M}: 1M unpaired text samples. \fancynumber{4} \textbf{ReVision-2M:} 2M unpaired text samples. Comparing the text-only methods, ReVision-2M ($49.75$) dramatically outperforms Unicorn ($43.94$). Since Unicorn relies on simple mean-shifting, its synthesized features lack the correct geometric shape, leading to poor manifold penetration. ReVision generates pseudo-features that statistically mimic real visual distributions, proving that how you align matters as much as what you align. Additionally, we observe that Unicorn achieves notably high scores on Reasoning benchmarks. This discrepancy stems from the SFT data distribution: the InternVL-Chat-V1-2-SFT used in our controlled experiments focuses primarily on general visual conversation and basic perception, whereas Unicorn's original SFT dataset incorporates a massive volume of complex reasoning tasks. This aligns with our perspective that unlocking deep reasoning potential relies heavily on complex SFT tasks, suggesting that ReVision's performance could be further elevated with more challenging instruction tuning data. Importantly, ReVision-2M (49.75) surpasses the w/. Image baseline (48.91) trained on 1M real paired samples. Comparing with paired data, ReVision-2M (49.75) outperforms the 1M paired baseline (48.91) at only 74\% of the cost (0.74 vs. 1.00). This highlights a promising scaling trajectory: continuously scaling up low-cost text data can surpass the performance of expensive paired data with significantly lower overhead. 



\section{Conclusion}

We address the Modality Gap by establishing the Fixed-frame Theory. Moving beyond isotropic assumptions, we decompose the gap into stable biases and anisotropic residuals, revealing it as a structured geometric phenomenon rather than random noise. Guided by these insights, we introduce ReAlign, a training-free statistical alignment strategy, and ReVision, a scalable paradigm leveraging unpaired text as a substitute for expensive image-text pairs. Experiments show ReVision significantly mitigates the gap and outperforms baselines, proving that high-quality visual structures can be learned from pure text. This work offers a robust, cost-effective pathway for scaling MLLMs.


\bibliography{ref}
\bibliographystyle{plainnat}


\newpage
\appendix
\onecolumn

\newpage

\section{Modality Gap Phenomenon: Structural Causes}
\label{sec:appendix_gap_causes}

This appendix explains why a modality gap can persist in the contrastive dual-encoder setting studied in Sec.~\ref{sec:modality_gap}. The key point is that the gap is not merely a consequence of noisy data or imperfect optimization. Rather, it is enabled by the interaction of three structural factors: dual-encoder isolation, spherical normalization, and dot-product contrastive supervision. Together, these factors create a weakly controlled channel in which the orthogonal component of the gap can survive late into training.

\paragraph{Recall.}
For paired embeddings \(e_x(t),e_y(t)\in\mathbb{R}^d\), the modality gap is defined as \(\Delta(t):=e_x(t)-e_y(t)\), with mean gap \(m(t):=\mathbb{E}[\Delta(t)]\). Under the fixed reference decomposition \(\mathbb{R}^d=U\oplus V\), Sec.~\ref{subsec:decomposition_framework} decomposes the gap as
\[
\Delta(t)=\beta(t)+\gamma(t)+\delta(t)+\zeta(t),
\]
where \(\beta(t)=P_Um(t)\) is the principal mean bias in the dominant representation subspace, \(\gamma(t)=P_Vm(t)\) is the passive orthogonal bias in the orthogonal complement, and \(\delta(t),\zeta(t)\) are zero-mean residual components. Sec.~\ref{subsec:bias} shows that \(\gamma(t)\) is not rapidly eliminated in late training, while Sec.~\ref{subsec:res} shows that the residual components have anisotropic second-order structure. Sec.~\ref{subsec:phantom_drift} further shows that spherical normalization couples the mean anchor with the residual covariance, producing Phantom Drift.

\paragraph{Thesis.}
The persistence of the modality gap is explained by three structural causes. First, dual-encoder isolation creates independent modality anchors before alignment. Second, spherical normalization turns second-order residual anisotropy into centroid drift on the unit hypersphere. Third, the dot-product InfoNCE head restricts embedding-level gradients to the contrastive-set span, leaving directions outside the dominant representation subspace weakly corrected.

\subsection{Cause 1: Dual-Encoder Isolation}

\paragraph{Structural condition.}
The model uses two separate encoders for the two modalities. Each modality is processed independently until the final shared embedding space, producing unit-normalized embeddings \(e_x(t)\) and \(e_y(t)\).

\paragraph{Mechanism.}
Because the two encoders have different input structures, architectures, and inductive biases, their representation distributions need not have the same centroid. Thus, even when paired samples are semantically aligned, the population mean gap \(m(t)=\mathbb{E}[e_x(t)-e_y(t)]\) is generally nonzero. Under the fixed reference frame, this mean mismatch decomposes into \(\beta(t)=P_Um(t)\) and \(\gamma(t)=P_Vm(t)\). The component \(\beta(t)\) is the anchor displacement inside the dominant representation subspace, while \(\gamma(t)\) is the orthogonal mean displacement that can persist when direct optimization pressure along \(V\) is weak.

\paragraph{Implication.}
Dual-encoder isolation is therefore the structural source of the first-order anchor mismatch. If the modalities were fused before the final embedding layer, their intermediate features would be jointly processed and the notion of two independently formed modality anchors would be substantially weakened. In contrast, the dual-encoder design allows each modality distribution to develop its own centroid before contrastive alignment acts on the final embeddings.

\subsection{Cause 2: Spherical Normalization}

\paragraph{Structural condition.}
The encoder outputs are explicitly normalized by an L2 projection \(e=z/\|z\|\), so all final embeddings lie on the unit hypersphere.

\paragraph{Mechanism.}
Spherical normalization is nonlinear, so Euclidean mean alignment does not imply spherical centroid alignment. Let a mean-aligned source representation be \(z=\mu+\epsilon\), where \(\mathbb{E}[\epsilon]=0\) and \(\operatorname{Cov}(\epsilon)=\Sigma\). Although \(\mathbb{E}[z]=\mu\), the normalized centroid is
\[
\mu^\pi=\mathbb{E}\!\left[\frac{\mu+\epsilon}{\|\mu+\epsilon\|}\right],
\]
which generally differs from \(\pi(\mu)=\mu/\|\mu\|\). Thus the normalization-induced drift is
\[
\Delta^\pi
=
\mathbb{E}\!\left[\frac{\mu+\epsilon}{\|\mu+\epsilon\|}\right]
-
\frac{\mu}{\|\mu\|}.
\]
This is the Phantom Drift analyzed in Sec.~\ref{subsec:phantom_drift}. Its magnitude and direction depend on both the anchor direction \(\mu\) and the residual covariance \(\Sigma\). Therefore, the anisotropic residual structure identified in Sec.~\ref{subsec:res} is not geometrically harmless: after normalization, it can shift the spherical centroid even when the Euclidean mean has already been aligned.

\paragraph{Implication.}
Spherical normalization turns second-order residual shape into a first-order centroid effect on the hypersphere. Without this nonlinear projection, mean correction would remain a linear Euclidean operation. With normalization, however, residual anisotropy and anchor direction interact, producing an additional centroid mismatch that must be corrected separately.

\subsection{Cause 3: Dot-Product Contrastive Supervision}

\paragraph{Structural condition.}
The contrastive objective uses dot-product similarity between unit-normalized embeddings, as in InfoNCE. The loss compares each embedding only through its inner products with embeddings from the current contrastive set.

\paragraph{Mechanism.}
For a mini-batch \(B_t\), let
\[
U_t^{(B)}
:=
\operatorname{span}
\left(
\{e_{x,i}(t)\}_{i\in B_t}
\cup
\{e_{y,i}(t)\}_{i\in B_t}
\right)
\]
be the contrastive-set span. For the mini-batch InfoNCE loss, the embedding-level gradients satisfy
\[
g_{x,i}^{(B)}(t)\in\operatorname{span}\{e_{y,j}(t):j\in B_t\},
\qquad
g_{y,i}^{(B)}(t)\in\operatorname{span}\{e_{x,j}(t):j\in B_t\}.
\]
Hence, the gradients are constrained by the representation span of the current contrastive set. In late training, embeddings concentrate near the instantaneous dominant subspace \(U_t\), so the main optimization signal is also concentrated near \(U_t\). When we measure the leakage into the fixed orthogonal complement \(V=U^\perp\), any vector \(g\in U_t\) satisfies
\[
\frac{\|P_Vg\|}{\|g\|}
\le
\|P_VP_{U_t}\|
=
\sin\theta(U_t,U),
\]
where \(\theta(U_t,U)\) is the largest principal angle between the instantaneous subspace \(U_t\) and the fixed reference subspace \(U\). Thus, a large portion of the observed \(V\)-direction leakage can be attributed to subspace rotation rather than to a strong corrective signal along \(V\).

\paragraph{Implication.}
The dot-product InfoNCE head does not uniformly penalize mismatch in all ambient directions. Instead, it primarily produces gradients along directions spanned by the current contrastive representations. As a result, the orthogonal component \(\gamma(t)\) is only weakly corrected in late training. This explains why the POB behaves as a passive bias: it maintains a stable direction and drifts slowly rather than being rapidly optimized away.

\subsection{Summary}

The modality gap persists because standard contrastive dual-encoder training contains three interacting structural mechanisms. Dual-encoder isolation creates independent modality anchors, giving rise to a nonzero first-order gap \(m(t)=\beta(t)+\gamma(t)\). Spherical normalization couples this anchor with anisotropic residual covariance, producing Phantom Drift on the unit hypersphere. Dot-product InfoNCE supervision restricts embedding-level gradients to the contrastive-set span, leaving the orthogonal complement \(V\) weakly corrected. Therefore, the modality gap should not be viewed as a mere optimization failure or random noise artifact. It is a structural byproduct of the representation geometry induced by dual encoders, spherical embeddings, and contrastive dot-product training, and thus requires explicit geometric rectification.

\section{Proof of Theoretical Claims}
\label{app:proofs}

This appendix formalizes the mathematical claims used in Sec.~\ref{sec:modality_gap} and Sec.~\ref{sec:realign}. We separate deterministic identities from empirical conditions. In particular, low reference leakage, high effective condition number, and stable spectral profiles are empirical observations measured in the main text; the results below prove what these observations imply under the fixed-reference decomposition.

\subsection{Preliminaries}

\begin{definition}[Orthogonal fixed reference frame]
Let \(U\subset\mathbb{R}^d\) be the dominant representation subspace estimated at the reference time \(t_0\), and let \(V:=U^\perp\). Denote by \(P_U\) and \(P_V\) the orthogonal projectors onto \(U\) and \(V\), respectively. Then \(P_U+P_V=I\), \(P_UP_V=0\), and every vector \(a\in\mathbb{R}^d\) admits the unique orthogonal decomposition \(a=P_Ua+P_Va\).
\end{definition}

\begin{definition}[Modality gap and residual components]
For paired embeddings \(e_x(t),e_y(t)\in\mathbb{R}^d\), define the modality gap as \(\Delta(t):=e_x(t)-e_y(t)\) and its mean as \(m(t):=\mathbb{E}[\Delta(t)]\). Under the fixed frame \(U\oplus V\), define
\[
\beta(t):=P_Um(t),\qquad \gamma(t):=P_Vm(t),\qquad
\delta(t):=P_U(\Delta(t)-m(t)),\qquad
\zeta(t):=P_V(\Delta(t)-m(t)).
\]
\end{definition}

\subsection{Proof of the Bias--Residual Decomposition}

\begin{lemma}[Orthogonal bias--residual decomposition]
For every \(t\ge t_0\), the modality gap decomposes as
\[
\Delta(t)=\beta(t)+\gamma(t)+\delta(t)+\zeta(t).
\]
Moreover, \(\mathbb{E}[\delta(t)]=0\), \(\mathbb{E}[\zeta(t)]=0\), \(\delta(t)\in U\), \(\zeta(t)\in V\), and \(\langle \delta(t),\zeta(t)\rangle=0\) almost surely.
\end{lemma}

\begin{proof}
Since \(P_U+P_V=I\), we have
\[
\Delta(t)=m(t)+(\Delta(t)-m(t))
        =P_Um(t)+P_Vm(t)+P_U(\Delta(t)-m(t))+P_V(\Delta(t)-m(t)).
\]
By definition, this is exactly \(\beta(t)+\gamma(t)+\delta(t)+\zeta(t)\). Since \(P_U\) and \(P_V\) are linear,
\[
\mathbb{E}[\delta(t)]
=
P_U\mathbb{E}[\Delta(t)-m(t)]
=0,
\qquad
\mathbb{E}[\zeta(t)]
=
P_V\mathbb{E}[\Delta(t)-m(t)]
=0.
\]
The membership \(\delta(t)\in U\) and \(\zeta(t)\in V\) follows from the definitions of the projectors. Since \(U\perp V\), the two residual components are orthogonal almost surely.
\end{proof}

\begin{proposition}[Mean centering is necessary for pure residual covariance]
Let \(\Delta_U(t):=P_U\Delta(t)=\beta(t)+\delta(t)\). Then
\[
\mathbb{E}\!\left[\Delta_U(t)\Delta_U(t)^\top\right]
=
\beta(t)\beta(t)^\top+\operatorname{Cov}(\delta(t)).
\]
Thus, without removing \(\beta(t)\), the second moment in \(U\) mixes the first-order mean displacement with the second-order residual covariance.
\end{proposition}

\begin{proof}
Since \(\mathbb{E}[\delta(t)]=0\), we have
\[
\mathbb{E}\!\left[(\beta(t)+\delta(t))(\beta(t)+\delta(t))^\top\right]
=
\beta(t)\beta(t)^\top
+
\mathbb{E}[\delta(t)\delta(t)^\top].
\]
Because \(\delta(t)\) is zero-mean, \(\mathbb{E}[\delta(t)\delta(t)^\top]=\operatorname{Cov}(\delta(t))\). This proves the claim.
\end{proof}

\begin{proposition}[Translation-correctability of the first-order anchor mismatch]
At the Euclidean mean level, the first-order anchor mismatch \(m(t)=\beta(t)+\gamma(t)\) is removed by the translation \(\Delta(t)\mapsto \Delta(t)-m(t)\). Equivalently, translating one modality by the difference of modality means aligns the two modality anchors.
\end{proposition}

\begin{proof}
Let \(e_y\) be the source-modality embedding and \(e_x\) the target-modality embedding, with means \(\mu_y=\mathbb{E}[e_y]\) and \(\mu_x=\mathbb{E}[e_x]\). Define the translated source embedding \(\dot e_y=(e_y-\mu_y)+\mu_x\). Then
\[
\mathbb{E}[\dot e_y]=\mathbb{E}[e_y]-\mu_y+\mu_x=\mu_x.
\]
Thus the source anchor is translated to the target anchor. In terms of the gap, subtracting \(m(t)\) gives \(\mathbb{E}[\Delta(t)-m(t)]=0\), so the first-order mean mismatch is removed.
\end{proof}

\subsection{Gradient Concentration and Passive Orthogonal Bias}

\begin{lemma}[InfoNCE gradient span constraint]
Consider a mini-batch \(B_t\) with normalized embeddings \(\{e_{x,i}(t)\}_{i\in B_t}\) and \(\{e_{y,i}(t)\}_{i\in B_t}\), and let the image-to-text InfoNCE loss for sample \(i\) be
\[
\ell_i^{x\to y}
=
-\log
\frac{\exp(\langle e_{x,i},e_{y,i}\rangle/\tau)}
{\sum_{j\in B_t}\exp(\langle e_{x,i},e_{y,j}\rangle/\tau)}.
\]
Then \(\nabla_{e_{x,i}}\ell_i^{x\to y}\in\operatorname{span}\{e_{y,j}:j\in B_t\}\). Similarly, for the text-to-image loss, the gradient with respect to \(e_{y,i}\) lies in \(\operatorname{span}\{e_{x,j}:j\in B_t\}\). Consequently, the embedding-level gradients of the symmetric InfoNCE loss are constrained by the contrastive-set span.
\end{lemma}

\begin{proof}
Let
\[
p_{ij}
=
\frac{\exp(\langle e_{x,i},e_{y,j}\rangle/\tau)}
{\sum_{k\in B_t}\exp(\langle e_{x,i},e_{y,k}\rangle/\tau)}.
\]
Differentiating \(\ell_i^{x\to y}\) with respect to \(e_{x,i}\) gives
\[
\nabla_{e_{x,i}}\ell_i^{x\to y}
=
\frac{1}{\tau}\left(\sum_{j\in B_t}p_{ij}e_{y,j}-e_{y,i}\right),
\]
which is a linear combination of \(\{e_{y,j}:j\in B_t\}\). The text-to-image direction is symmetric. For the symmetric InfoNCE objective, gradients are sums of such terms, and hence remain within the corresponding contrastive-set span.
\end{proof}

\begin{lemma}[Geometric leakage bound]
Let \(U_t\) be an instantaneous \(r\)-dimensional subspace and \(U\) the fixed \(r\)-dimensional reference subspace, with \(V=U^\perp\). Let \(\theta(U_t,U)\) be their largest principal angle. For any nonzero \(g\in U_t\),
\[
\frac{\|P_Vg\|}{\|g\|}
\le
\|P_VP_{U_t}\|_2
=
\sin\theta(U_t,U).
\]
\end{lemma}

\begin{proof}
Since \(g\in U_t\), we have \(P_{U_t}g=g\). Therefore,
\[
\|P_Vg\|=\|P_VP_{U_t}g\|\le \|P_VP_{U_t}\|_2\|g\|.
\]
Dividing by \(\|g\|\) gives the inequality. The equality \(\|P_VP_{U_t}\|_2=\sin\theta(U_t,U)\) is the standard relation between orthogonal projection residuals and the largest principal angle between two subspaces.
\end{proof}

\begin{proposition}[Weak \(V\)-direction updates imply slow POB drift]
Suppose the orthogonal mean bias evolves as \(\gamma(t+1)=\gamma(t)+\eta(t)\), where \(\eta(t)\in V\) is the effective update along \(V\). If \(\|\eta(t)\|\le \varepsilon_t\), then over a window \(\mathcal{T}=\{t_0,\ldots,t_0+\tau-1\}\),
\[
\|\gamma(t_0+k)-\gamma(t_0)\|
\le
\sum_{j=0}^{k-1}\varepsilon_{t_0+j}.
\]
Moreover, if \(\|\eta(t)\|\le \alpha\|\gamma(t)\|\) with \(0<\alpha<1\), then the adjacent-step direction changes are small, in the sense that
\[
\cos(\gamma(t+1),\gamma(t))
\ge
1-\frac{2\alpha^2}{(1-\alpha)^2}.
\]
\end{proposition}

\begin{proof}
The drift bound follows by telescoping:
\[
\gamma(t_0+k)-\gamma(t_0)=\sum_{j=0}^{k-1}\eta(t_0+j),
\]
and applying the triangle inequality. For the cosine bound, let \(a=\gamma(t)\) and \(b=a+\eta(t)\). Since \(\|\eta(t)\|\le\alpha\|a\|\), we have \(\|b\|\ge (1-\alpha)\|a\|\). The normalized direction perturbation satisfies
\[
\left\|\frac{b}{\|b\|}-\frac{a}{\|a\|}\right\|
\le
\frac{2\|\eta(t)\|}{\|b\|}
\le
\frac{2\alpha}{1-\alpha}.
\]
Using \(\|u-v\|^2=2(1-\langle u,v\rangle)\) for unit vectors \(u,v\), we obtain
\[
1-\cos(a,b)
\le
\frac{1}{2}\left(\frac{2\alpha}{1-\alpha}\right)^2
=
\frac{2\alpha^2}{(1-\alpha)^2}.
\]
This proves the claim.
\end{proof}

\subsection{Residual Covariance and Anisotropic Shape}

\begin{lemma}[Step-wise covariance avoids temporal drift mixing]
Let \(X_s=\mu_s+R_s\) for \(s=0,\ldots,\tau-1\), where \(\mathbb{E}[R_s]=0\) and \(\operatorname{Cov}(R_s)=\Sigma_s\). If \(S\) is uniformly sampled from \(\{0,\ldots,\tau-1\}\), then the covariance of the temporally pooled variable \(X_S\) satisfies
\[
\operatorname{Cov}(X_S)
=
\frac{1}{\tau}\sum_{s=0}^{\tau-1}\Sigma_s
+
\operatorname{Cov}(\mu_S).
\]
Thus, directly pooling samples across time includes an additional covariance term induced by the drift of the time-varying mean.
\end{lemma}

\begin{proof}
By the law of total covariance,
\[
\operatorname{Cov}(X_S)
=
\mathbb{E}_S[\operatorname{Cov}(X_S\mid S)]
+
\operatorname{Cov}_S(\mathbb{E}[X_S\mid S]).
\]
Since \(\operatorname{Cov}(X_S\mid S=s)=\Sigma_s\) and \(\mathbb{E}[X_S\mid S=s]=\mu_s\), the identity follows.
\end{proof}

\begin{proposition}[Residual anisotropy is a covariance-shape property]
Let \(R\) be a zero-mean residual supported on a subspace \(W\) with covariance \(\Sigma_W\). If \(\Sigma_W=\sigma^2P_W\) for some \(\sigma^2\ge 0\), then \(R\) is isotropic in \(W\). Conversely, if the spectrum of \(\Sigma_W\) is non-uniform, then the residual has an anisotropic ellipsoidal shape in \(W\). In particular, any spectral stretching measure \(\kappa_{\mathrm{eff}}\) that equals \(1\) on trace-matched isotropic covariances and exceeds \(1\) on non-uniform spectra certifies anisotropy when \(\kappa_{\mathrm{eff}}(\Sigma_W)>1\).
\end{proposition}

\begin{proof}
For a zero-mean random vector, the covariance determines the second-order shape of its distribution. If \(\Sigma_W=\sigma^2P_W\), then all directions in \(W\) have the same variance:
\[
\operatorname{Var}(\langle u,R\rangle)=u^\top\Sigma_Wu=\sigma^2
\]
for every unit vector \(u\in W\). Hence the residual is second-order isotropic in \(W\). If the eigenvalues of \(\Sigma_W\) are not all equal, then there exist two unit eigen-directions \(u,v\in W\) with different variances, so the residual energy is direction-dependent. The covariance ellipsoid is therefore elongated along higher-variance directions, which is exactly second-order anisotropy.
\end{proof}

\begin{lemma}[Mean direction and residual principal direction are distinct objects]
Let \(\gamma\in V\) be the mean bias direction and let \(v_1\) be the first principal direction of a residual covariance \(\Sigma_V\). If \(|\langle \gamma/\|\gamma\|,v_1\rangle|\) is small, then the first-order mean displacement and the leading second-order residual spread are geometrically decoupled in \(V\).
\end{lemma}

\begin{proof}
The vector \(\gamma\) describes the first moment, i.e., the centroid displacement in \(V\). The vector \(v_1\) is an eigenvector of \(\Sigma_V\) associated with its largest eigenvalue, and therefore describes the direction of maximum residual variance. If these directions are nearly orthogonal, then the direction of the mean displacement does not coincide with the direction of largest residual energy. Hence the first-order and second-order structures are geometrically distinct.
\end{proof}

\subsection{Phantom Drift Induced by Spherical Normalization}

\begin{definition}[Spherical projection and spherical centroid]
For \(z\in\mathbb{R}^d\setminus\{0\}\), define the spherical projection \(\pi(z):=z/\|z\|\). For a random vector \(Z\), define its spherical centroid as \(\mathbb{E}[\pi(Z)]\), whenever the expectation exists.
\end{definition}

\begin{theorem}[Euclidean mean alignment does not imply spherical centroid alignment]
Let \(Z=\mu+\epsilon\), where \(\mu\neq 0\), \(\mathbb{E}[\epsilon]=0\), and \(\operatorname{Cov}(\epsilon)=\Sigma\). Even though \(\mathbb{E}[Z]=\mu\), in general
\[
\mathbb{E}[\pi(Z)]\neq \pi(\mu).
\]
Therefore, Euclidean mean alignment does not guarantee centroid alignment after spherical normalization.
\end{theorem}

\begin{proof}
The map \(\pi(z)=z/\|z\|\) is nonlinear. Therefore, expectation and projection do not commute in general:
\[
\mathbb{E}[\pi(Z)]\ne \pi(\mathbb{E}[Z]).
\]
Since \(\mathbb{E}[Z]=\mu\), the right-hand side is \(\pi(\mu)\). A concrete example suffices. Let \(d=2\), \(\mu=(1,0)^\top\), and let \(\epsilon\) take values \((0,a)^\top\) and \((0,-a)^\top\) with probability \(1/2\), where \(a>0\). Then \(\mathbb{E}[\epsilon]=0\), but
\[
\mathbb{E}[\pi(\mu+\epsilon)]
=
\left(\frac{1}{\sqrt{1+a^2}},0\right)^\top
\ne
(1,0)^\top
=
\pi(\mu).
\]
Thus spherical centroid alignment is not guaranteed by Euclidean mean alignment.
\end{proof}

\begin{theorem}[Second-order expansion of Phantom Drift]
Let \(Z=\mu+\epsilon\), where \(\mu\ne 0\), \(\mathbb{E}[\epsilon]=0\), and \(\operatorname{Cov}(\epsilon)=\Sigma\). Let \(r:=\|\mu\|\) and \(a:=\mu/r\). For sufficiently small residual magnitude, the normalization-induced centroid drift
\[
\Delta^\pi
:=
\mathbb{E}[\pi(\mu+\epsilon)]-\pi(\mu)
\]
admits the second-order approximation
\[
\Delta^\pi
=
-\frac{1}{r^2}\Sigma a
-\frac{1}{2r^2}a\,\operatorname{tr}(\Sigma)
+\frac{3}{2r^2}a\,(a^\top\Sigma a)
+
O(\mathbb{E}\|\epsilon\|^3/r^3).
\]
In particular, the tangential component of the drift is
\[
(I-aa^\top)\Delta^\pi
=
-\frac{1}{r^2}(I-aa^\top)\Sigma a
+
O(\mathbb{E}\|\epsilon\|^3/r^3).
\]
Thus anisotropic residual covariance can shift the spherical centroid away from the anchor direction.
\end{theorem}

\begin{proof}
Let \(f(z)=z/\|z\|\). The first derivative at \(\mu\) is
\[
Df_\mu[h]
=
\frac{1}{r}(I-aa^\top)h.
\]
Since \(\mathbb{E}[\epsilon]=0\), the first-order term vanishes after taking expectation. The second derivative satisfies
\[
D^2f_\mu[h,h]
=
-\frac{2}{r^3}h(\mu^\top h)
-\frac{1}{r^3}\mu\|h\|^2
+\frac{3}{r^5}\mu(\mu^\top h)^2.
\]
Applying Taylor expansion around \(\mu\) and taking expectation gives
\[
\mathbb{E}[f(\mu+\epsilon)]
=
f(\mu)
+
\frac{1}{2}\mathbb{E}[D^2f_\mu[\epsilon,\epsilon]]
+
O(\mathbb{E}\|\epsilon\|^3/r^3).
\]
Using \(\mathbb{E}[\epsilon(\mu^\top\epsilon)]=\Sigma\mu\), \(\mathbb{E}\|\epsilon\|^2=\operatorname{tr}(\Sigma)\), and \(\mathbb{E}[(\mu^\top\epsilon)^2]=\mu^\top\Sigma\mu\), we obtain
\[
\Delta^\pi
=
-\frac{1}{r^3}\Sigma\mu
-\frac{1}{2r^3}\mu\,\operatorname{tr}(\Sigma)
+
\frac{3}{2r^5}\mu(\mu^\top\Sigma\mu)
+
O(\mathbb{E}\|\epsilon\|^3/r^3).
\]
Substituting \(\mu=ra\) yields the claimed expression. Finally, projecting onto the tangent space \(a^\perp\) removes the terms parallel to \(a\), leaving
\[
(I-aa^\top)\Delta^\pi
=
-\frac{1}{r^2}(I-aa^\top)\Sigma a
+
O(\mathbb{E}\|\epsilon\|^3/r^3).
\]
This proves the result.
\end{proof}

\begin{corollary}[Role of anisotropy]
If \(\Sigma=\sigma^2I\), then the leading-order tangential drift vanishes. If \((I-aa^\top)\Sigma a \neq 0\), then the leading-order tangential drift is nonzero. Therefore, anisotropy that is not aligned with the anchor direction induces angular centroid drift on the unit hypersphere.
\end{corollary}

\begin{proof}
Substituting \(\Sigma=\sigma^2I\) into the tangential term gives \((I-aa^\top)\Sigma a=\sigma^2(I-aa^\top)a=0\). Conversely, if \((I-aa^\top)\Sigma a\ne 0\), then the leading-order tangential term in the previous theorem is nonzero.
\end{proof}

\subsection{Properties of the ReAlign Operator}

\begin{proposition}[Anchor Alignment matches the target mean]
Let \(\dot e_y=(e_y-\mu_y)+\mu_x\), where \(\mu_y=\mathbb{E}[e_y]\) and \(\mu_x=\mathbb{E}[e_x]\). Then \(\mathbb{E}[\dot e_y]=\mu_x\).
\end{proposition}

\begin{proof}
By linearity of expectation,
\[
\mathbb{E}[\dot e_y]
=
\mathbb{E}[e_y]-\mu_y+\mu_x
=
\mu_x.
\]
\end{proof}

\begin{proposition}[Trace Alignment matches residual energy and preserves spectral shape]
Let \(\tilde e_y=\mu_x+s(e_y-\mu_y)\), where \(s=\sqrt{T_x/(T_y+\varepsilon)}\), \(T_y=\operatorname{tr}(\Sigma_y)\), and \(T_x=\operatorname{tr}(\Sigma_x)\). Then
\[
\operatorname{Cov}(\tilde e_y)=s^2\Sigma_y,
\qquad
\operatorname{tr}(\operatorname{Cov}(\tilde e_y))=s^2T_y\approx T_x.
\]
Moreover, scalar scaling preserves the eigenvectors and the trace-normalized spectrum of \(\Sigma_y\).
\end{proposition}

\begin{proof}
Since \(\tilde e_y-\mu_x=s(e_y-\mu_y)\), we have
\[
\operatorname{Cov}(\tilde e_y)
=
s^2\operatorname{Cov}(e_y-\mu_y)
=
s^2\Sigma_y.
\]
Taking traces gives \(\operatorname{tr}(\operatorname{Cov}(\tilde e_y))=s^2T_y\), which equals \(T_x\) up to the stabilizing constant \(\varepsilon\). If \(\Sigma_y=Q\Lambda Q^\top\), then \(s^2\Sigma_y=Q(s^2\Lambda)Q^\top\). Thus eigenvectors are unchanged, and
\[
\frac{\lambda(s^2\Sigma_y)}{\operatorname{tr}(s^2\Sigma_y)}
=
\frac{s^2\lambda(\Sigma_y)}{s^2\operatorname{tr}(\Sigma_y)}
=
\frac{\lambda(\Sigma_y)}{\operatorname{tr}(\Sigma_y)}.
\]
Therefore Trace Alignment changes only the global residual energy scale while preserving the trace-normalized spectral shape.
\end{proof}

\begin{proposition}[Centroid Alignment removes the measured spherical centroid shift before the final projection]
Let \(e'_y=\pi(\tilde e_y)\), \(\mu'_y=\mathbb{E}[e'_y]\), and \(e''_y=e'_y-\mu'_y+\mu_x\). Then \(\mathbb{E}[e''_y]=\mu_x\).
\end{proposition}

\begin{proof}
Again by linearity of expectation,
\[
\mathbb{E}[e''_y]
=
\mathbb{E}[e'_y]-\mu'_y+\mu_x
=
\mu_x.
\]
\end{proof}

\begin{proposition}[Residual centroid error after the final projection is controlled by the second translation]
Let \(e'_y\) be unit-normalized and let \(h:=\mu_x-\mu'_y\). Define \(e''_y=e'_y+h\) and \(\hat e_y=\pi(e''_y)\). If \(\|h\|<1\), then
\[
\|\mathbb{E}[\hat e_y]-\mu_x\|
\le
\|h\|.
\]
\end{proposition}

\begin{proof}
For any unit vector \(q\) and any \(h\) with \(\|h\|<1\), we have
\[
\|\pi(q+h)-(q+h)\|
=
\left|1-\|q+h\|\right|
\le
\|h\|,
\]
where the last inequality follows from the reverse triangle inequality. Taking \(q=e'_y\) and expectation gives
\[
\|\mathbb{E}[\hat e_y]-\mathbb{E}[e''_y]\|
\le
\mathbb{E}\|\hat e_y-e''_y\|
\le
\|h\|.
\]
Since \(\mathbb{E}[e''_y]=\mu_x\), the result follows.
\end{proof}

\subsection{Summary of Theoretical Implications}

The results above justify the three geometric levels used in the main text. First, the modality gap admits an orthogonal decomposition into first-order anchor displacement and zero-mean residual components. Second, after mean centering, the remaining residual shape is governed by its covariance spectrum, so anisotropy is a second-order property rather than a mean effect. Third, spherical normalization does not commute with expectation; it couples the anchor direction with the residual covariance and can create a secondary centroid shift. These facts motivate ReAlign: Anchor Alignment removes the Euclidean mean mismatch, Trace Alignment matches the global residual energy while preserving spectral shape, and Centroid Alignment corrects the measured spherical centroid drift.

\section{U--V Weak Coupling Analysis}
\label{sec:appendix_weak_coupling}

This appendix provides a diagnostic analysis of weak coupling between the dominant representation subspace \(U\) and its orthogonal complement \(V\). In Sec.~\ref{subsec:bias}, the reference leakage ratio \(\operatorname{leak}_{\mathrm{ref}}(t)\) is compared with the geometric baseline \(\sin\theta(U_t,U)\), where \(U_t\) is the instantaneous dominant subspace and \(U\) is the fixed reference subspace. If the embedding-level gradients were fully constrained to \(U_t\), then their projection onto the fixed complement \(V=U^\perp\) would be controlled by this principal-angle baseline. In practice, small deviations from this idealized geometric picture may appear. We model such deviations as weak second-order coupling between \(U\)-side and \(V\)-side residuals.

\subsection{Residual-Level Coupling Model}

Let \(\delta(t)\in U\) and \(\zeta(t)\in V\) denote the zero-mean residual components defined in Sec.~\ref{subsec:decomposition_framework}. To quantify whether fluctuations in \(U\) systematically explain fluctuations in \(V\), we introduce the following linear residual-level model:
\begin{equation}
    \zeta(t)=\mathbf{L}\delta(t)+\xi(t),
    \qquad
    \mathbf{L}:U\to V,
    \qquad
    \|\mathbf{L}\|_2\le \varepsilon_{UV}.
    \label{eq:uv_coupling_model}
\end{equation}
Here, \(\mathbf{L}\) is a weak coupling map from \(U\) to \(V\), and \(\xi(t)\in V\) is a zero-mean background residual that is uncorrelated with \(\delta(t)\). In reference-basis coordinates, \(\delta(t)\in\mathbb{R}^r\), \(\zeta(t)\in\mathbb{R}^{d-r}\), and \(\mathbf{L}\in\mathbb{R}^{(d-r)\times r}\).

\begin{lemma}[Second-moment identities under weak coupling]
\label{lem:uv_moment_identities}
Assume \(\mathbb{E}[\delta(t)]=0\), \(\mathbb{E}[\xi(t)]=0\), and \(\mathbb{E}[\xi(t)\delta(t)^\top]=0\). Let \(\Sigma_U:=\operatorname{Cov}(\delta(t))\) and \(\Sigma_\xi:=\operatorname{Cov}(\xi(t))\). Under Eq.~\eqref{eq:uv_coupling_model},
\begin{equation}
    \mathbb{E}[\zeta(t)\delta(t)^\top]=\mathbf{L}\Sigma_U,
    \qquad
    \operatorname{Cov}(\zeta(t))=\mathbf{L}\Sigma_U\mathbf{L}^\top+\Sigma_\xi.
    \label{eq:uv_moment_identities}
\end{equation}
Equivalently, \(\mathbb{E}[\delta(t)\zeta(t)^\top]=\Sigma_U\mathbf{L}^\top\).
\end{lemma}

\begin{proof}
Using \(\zeta=\mathbf{L}\delta+\xi\), we obtain
\[
\mathbb{E}[\zeta\delta^\top]
=
\mathbb{E}[(\mathbf{L}\delta+\xi)\delta^\top]
=
\mathbf{L}\mathbb{E}[\delta\delta^\top]
+
\mathbb{E}[\xi\delta^\top]
=
\mathbf{L}\Sigma_U.
\]
Similarly,
\[
\operatorname{Cov}(\zeta)
=
\mathbb{E}[(\mathbf{L}\delta+\xi)(\mathbf{L}\delta+\xi)^\top]
=
\mathbf{L}\Sigma_U\mathbf{L}^\top+\Sigma_\xi,
\]
where the cross terms vanish because \(\xi\) and \(\delta\) are uncorrelated. This proves the claim.
\end{proof}

\paragraph{Interpretation.}
The map \(\mathbf{L}\) does not replace the main \(U\oplus V\) decomposition. Rather, it measures whether the two residual components are statistically independent at the second-moment level. If \(\mathbf{L}=0\), then the \(V\)-side residual contains no linear component predictable from the \(U\)-side residual. If \(\|\mathbf{L}\|_2\) is small but nonzero, then the two residuals are weakly coupled.

\subsection{Estimating the Coupling Map}

Lemma~\ref{lem:uv_moment_identities} shows that \(\mathbf{L}\) is identifiable from second moments when \(\Sigma_U\) is nonsingular on its effective support. The population least-squares map is
\begin{equation}
    \mathbf{L}^\star
    =
    \mathbb{E}[\zeta(t)\delta(t)^\top]\Sigma_U^\dagger,
    \label{eq:uv_population_lstsq}
\end{equation}
where \(\Sigma_U^\dagger\) denotes the Moore--Penrose pseudoinverse.

In practice, we estimate \(\mathbf{L}\) in the reference bases of \(U\) and \(V\). Given probe residuals \(\{(\delta_i,\zeta_i)\}_{i=1}^n\), let \(\mathbf{D}:=[\delta_1,\ldots,\delta_n]\in\mathbb{R}^{r\times n}\) and \(\mathbf{Z}:=[\zeta_1,\ldots,\zeta_n]\in\mathbb{R}^{(d-r)\times n}\). We use the ridge estimator
\begin{equation}
    \widehat{\mathbf{L}}_\lambda
    =
    \mathbf{Z}\mathbf{D}^\top
    \left(
    \mathbf{D}\mathbf{D}^\top+\lambda I_r
    \right)^{-1}.
    \label{eq:uv_ridge_estimator}
\end{equation}
We report the spectral norm \(\|\widehat{\mathbf{L}}_\lambda\|_2\) as the coupling strength and the explained variance
\begin{equation}
    R^2_{\zeta\leftarrow\delta}
    =
    1-
    \frac{
    \|\mathbf{Z}-\widehat{\mathbf{L}}_\lambda\mathbf{D}\|_F^2
    }{
    \|\mathbf{Z}\|_F^2
    }.
    \label{eq:uv_explained_variance}
\end{equation}
A small \(R^2_{\zeta\leftarrow\delta}\) indicates that the \(V\)-side residual is largely independent of the \(U\)-side residual, while a nonzero value indicates weak predictable leakage from \(U\) to \(V\).

\subsection{Connection to Gradient Leakage}

The coupling model above is defined at the residual level. To connect it to the gradient leakage measured in Sec.~\ref{subsec:bias}, we state an explicit diagnostic assumption rather than treating it as a consequence of the residual model.

\paragraph{Gradient-level coupling assumption.}
For an embedding-level gradient \(g\), suppose its fixed-frame \(V\)-component admits the approximate decomposition
\begin{equation}
    P_Vg
    =
    P_VP_{U_t}g
    +
    \mathbf{L}P_Ug
    +
    r_g,
    \qquad
    \|r_g\|\le \rho\|g\|.
    \label{eq:uv_gradient_coupling}
\end{equation}
The first term is the geometric leakage caused by the rotation between \(U_t\) and \(U\); the second term is the weak coupling contribution; and \(r_g\) is an unexplained remainder. When \(\rho=0\), the gradient leakage is fully explained by subspace rotation and the coupling map.

\begin{lemma}[Geometric leakage baseline]
\label{lem:uv_geometric_baseline}
For any nonzero \(g\in U_t\),
\begin{equation}
    \frac{\|P_Vg\|}{\|g\|}
    \le
    \|P_VP_{U_t}\|_2
    =
    \sin\theta(U_t,U),
    \label{eq:uv_geometric_baseline}
\end{equation}
where \(\theta(U_t,U)\) is the largest principal angle between \(U_t\) and \(U\).
\end{lemma}

\begin{proof}
Since \(g\in U_t\), \(P_{U_t}g=g\). Therefore,
\[
\|P_Vg\|
=
\|P_VP_{U_t}g\|
\le
\|P_VP_{U_t}\|_2\|g\|.
\]
Dividing by \(\|g\|\) gives the inequality. The identity \(\|P_VP_{U_t}\|_2=\sin\theta(U_t,U)\) follows from the standard relation between projection residuals and principal angles.
\end{proof}

\begin{theorem}[Gradient leakage under weak \(U\)--\(V\) coupling]
\label{thm:uv_leakage_bound}
Under the gradient-level coupling assumption in Eq.~\eqref{eq:uv_gradient_coupling}, the leakage ratio satisfies
\begin{equation}
    \frac{\|P_Vg\|}{\|g\|}
    \le
    \sin\theta(U_t,U)
    +
    \|\mathbf{L}\|_2
    +
    \rho.
    \label{eq:uv_total_leakage_bound}
\end{equation}
In the idealized case \(\rho=0\), the leakage is bounded by the sum of the geometric subspace-rotation baseline and the weak coupling strength.
\end{theorem}

\begin{proof}
By Eq.~\eqref{eq:uv_gradient_coupling} and the triangle inequality,
\[
\|P_Vg\|
\le
\|P_VP_{U_t}g\|
+
\|\mathbf{L}P_Ug\|
+
\|r_g\|.
\]
The first term is bounded by Lemma~\ref{lem:uv_geometric_baseline}: \(\|P_VP_{U_t}g\|\le \sin\theta(U_t,U)\|g\|\). For the coupling term, \(\|\mathbf{L}P_Ug\|\le \|\mathbf{L}\|_2\|P_Ug\|\le \|\mathbf{L}\|_2\|g\|\). The remainder satisfies \(\|r_g\|\le \rho\|g\|\). Combining these inequalities and dividing by \(\|g\|\) proves the result.
\end{proof}

\subsection{Implication}

The weak coupling analysis should be interpreted as a diagnostic refinement of the main geometric picture. The dominant explanation for reference leakage is still the rotation between the instantaneous subspace \(U_t\) and the fixed reference subspace \(U\), captured by \(\sin\theta(U_t,U)\). The coupling map \(\mathbf{L}\) measures whether there is an additional linear pathway through which \(U\)-side residual fluctuations predict \(V\)-side residual fluctuations. When \(\|\widehat{\mathbf{L}}_\lambda\|_2\) and \(R^2_{\zeta\leftarrow\delta}\) are small, the \(U\)- and \(V\)-side residuals are weakly coupled, supporting the interpretation that the orthogonal bias \(\gamma(t)\) evolves passively rather than being aggressively corrected by direct \(V\)-direction optimization.

%

\section{Beyond the Isotropic Assumption: A Geometric Analysis}
\label{app:isotropic_ssumption}

In this section, we provide a comprehensive empirical verification of the alignment quality. While the Modality Gap metric measures the Euclidean distance between centroids, it fails to capture the structural fidelity of the aligned distribution. To address this, we examine the geometric properties from three perspectives.

\subsection{Spectral Analysis: Preserving Semantic Hierarchy}
\label{app:spectrum}

The eigenspectrum of the feature covariance matrix characterizes the intrinsic geometry of the data manifold. Deep representations typically exhibit a power-law decay ($\lambda_k \propto k^{-\alpha}$), indicating a rich semantic hierarchy where variance concentrates in principal directions.

To ensure a fair comparison, we normalize the global trace of all baselines to match the target modality before computing the spectrum. In Fig.~\ref{fig:geo_stats}(a), we plot the log-scaled singular values and quantify the geometric structure using the Power-Law Exponent ($\alpha$).

\textbf{Observation.} \fancynumber{1} \textbf{C$^3$} exhibits a significantly elevated tail and a flatter slope, with an $\alpha$ value of approximately 1.06, compared to $\alpha \approx 1.33$ for the original text. This confirms that injecting unstructured Gaussian noise reduces spectral anisotropy, producing a whitening effect that effectively dilutes the fine-grained semantic structure. \fancynumber{2} \textbf{ReAlign.} In contrast, ReAlign maintains a decay rate ($\alpha \approx 1.33$) that perfectly matches the intrinsic geometry of the source text. While the visual modality naturally possesses a lower rank ($\alpha \approx 1.83$), ReAlign achieves alignment without artificially distorting the text's inherent semantic hierarchy.

\begin{figure*}[t] 
  \centering
  \includegraphics[width=0.96\linewidth]{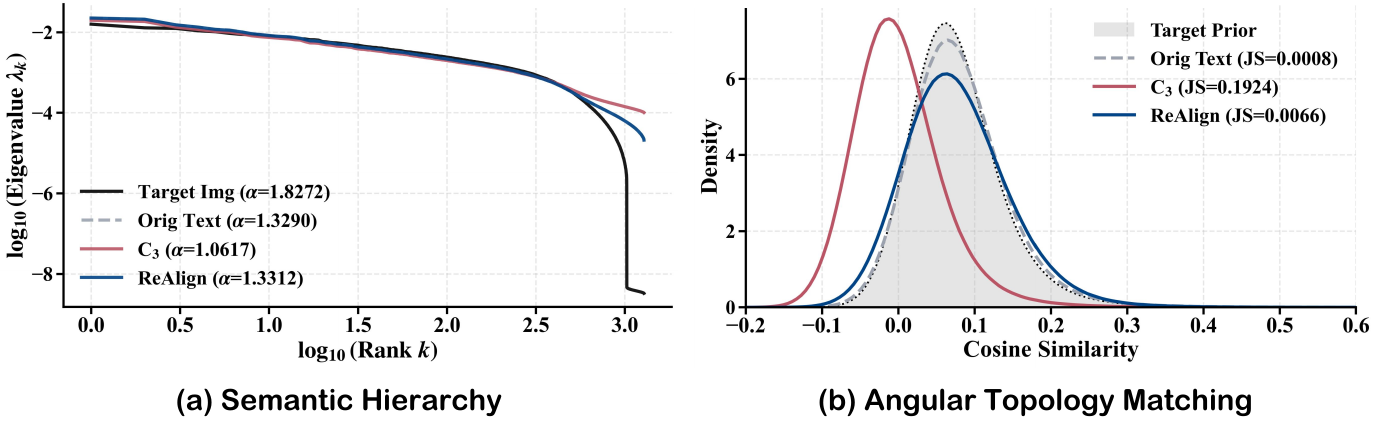}
  \caption{\textbf{Geometric Fidelity Analysis via Spectral and Angular Properties.} (a) Semantic Hierarchy: The eigenspectrum analysis reveals that C$^3$ (red line) exhibits a flattened slope with an elevated tail ($\alpha \approx 1.06$), indicating that unstructured noise injection dilutes fine-grained semantic structure. In contrast, ReAlign (blue line) maintains a power-law decay ($\alpha \approx 1.33$) that matches the intrinsic geometry of the source text. (b) Angular Topology Matching: KDE plots of cosine similarities demonstrate that C$^3$ causes a severe distributional shift (JS Divergence = 0.1924), destroying angular relationships. ReAlign achieves a near-perfect overlap with the target prior (JS Divergence = 0.0066), validating its ability to restore centroid alignment while preserving the topological structure.}
  \label{fig:geo_stats}
\end{figure*}

\subsection{Angular Distribution: Matching Cosine Topology}

We verify the angular alignment by computing the Kernel Density Estimate (KDE) of pairwise cosine similarities. We further quantify the distributional discrepancy using the Jensen-Shannon (JS) Divergence.

\textbf{Observation:} As shown in Fig.~\ref{fig:geo_stats}(b), the angular topology provides a rigorous test of structural preservation. \fancynumber{1} \textbf{C$^3$.} The noise injection results in a severe distributional shift, destroying the semantic relations on the hypersphere. Quantitative results show a high JS Divergence of 0.1904. \fancynumber{2} \textbf{ReAlign.} Our method achieves a JS Divergence as low as 0.0067, effectively overlapping with the target prior. This demonstrates that ReAlign restores centroid alignment while preserving the correct angular relationships between samples.

\subsection{Representation Visualization}
\label{app:pca}

\begin{figure*}[t] 
  \centering
  \includegraphics[width=0.98\linewidth]{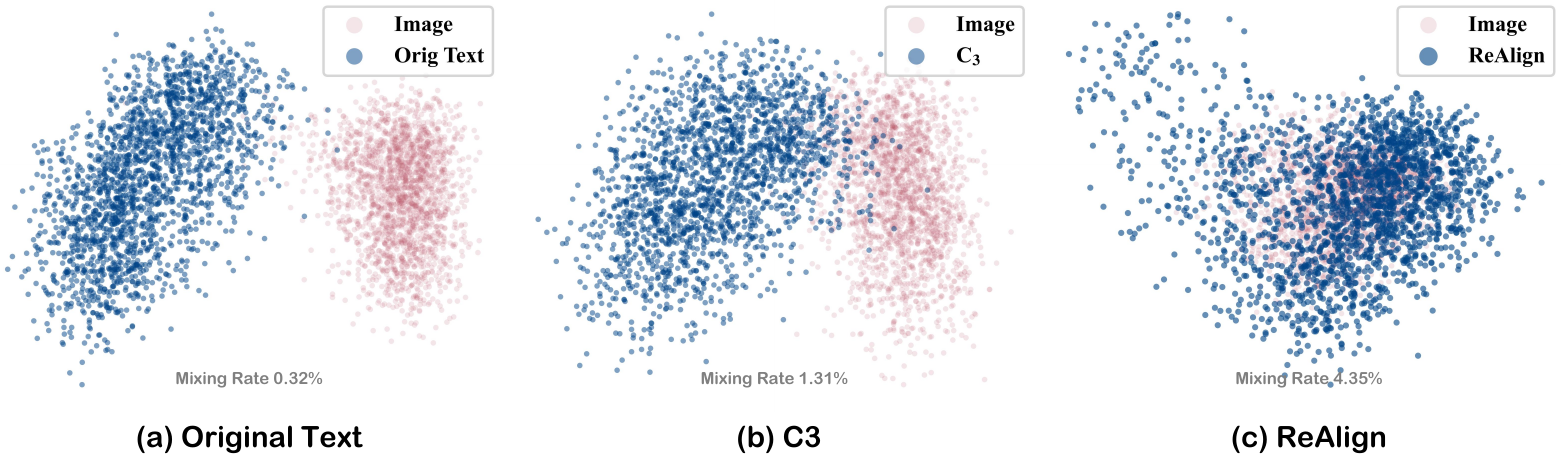}
  \caption{\textbf{Global Alignment Visualization via PCA.} We visualize the manifold alignment across three settings: (a) Original Text forms a distinct cluster separated from the image modality with negligible mixing (0.32\%). (b) C$^3$ expands the text distribution via noise but fails to effectively penetrate the visual manifold (1.31\%). (c) ReAlign successfully shifts the text distribution into the visual support region, achieving a mixing rate of 4.35\%. This represents a relative improvement of over $3\times$ compared to the C$^3$ baseline, confirming significant manifold penetration.}
  \label{fig:pca}
\end{figure*}

We use Principal Component Analysis (PCA) for a qualitative visualization of the global alignment. As shown in Fig.~\ref{fig:pca}. \fancynumber{1} \textbf{Qualitative Analysis:} The Original Text (Left) forms a cluster that is distinct and separated from the image modality. The C$^3$ baseline (Middle) expands the text distribution via noise but fails to penetrate the visual manifold. ReAlign (Right) effectively shifts the text distribution into the visual support region. \fancynumber{2} \textbf{Quantitative Analysis:} We validate this mixing in the high-dimensional space ($D=1280$) using the k-NN Mixing Rate with $k=20$. The unaligned text shows negligible mixing, at only 0.32\%. While C$^3$ slightly improves this to 1.31\%, ReAlign achieves a mixing rate of 4.35\%. This represents a relative improvement of over 3$\times$ compared to the strong C$^3$ baseline, and an improvement of over 10$\times$ compared to the unaligned state, confirming significant manifold penetration.

\section{Robustness Analysis}
\label{app:robustness_analysis}

In this section, we systematically evaluate the stability and scalability of ReAlign along three critical dimensions: \fancynumber{1} sample complexity for estimating alignment statistics, \fancynumber{2} numerical stability and computational efficiency, and \fancynumber{3} sensitivity to domain shifts.

\subsection{Data Efficiency \& Stability}
\label{app:sample_complexity}

\begin{wrapfigure}{r}{0.5\textwidth} 
  \centering
  \vspace{-30pt} 
  \includegraphics[width=0.9\linewidth]{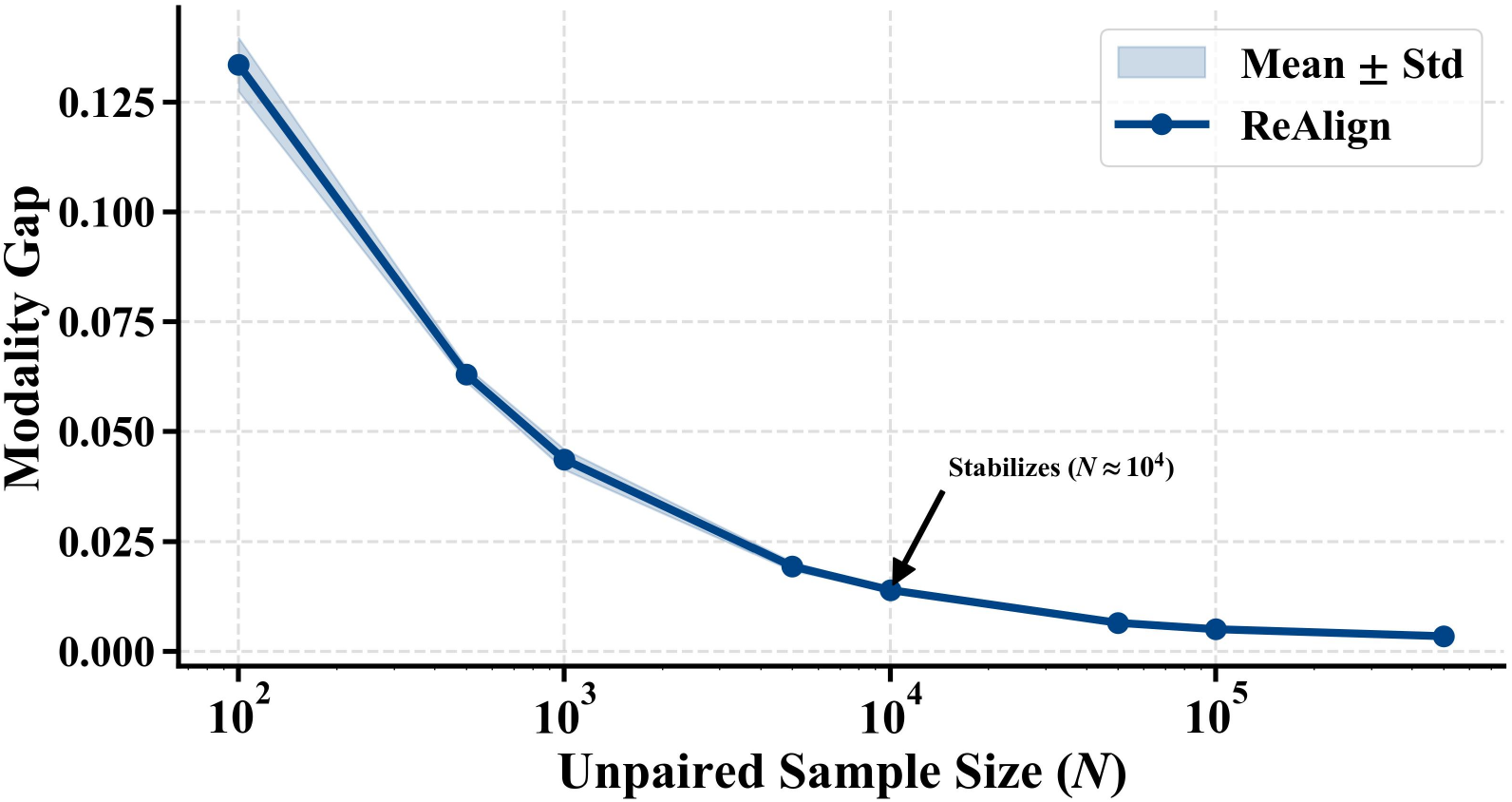}
  \caption{Impact of unpaired sample size on the modality gap for statistical estimation.}
  \label{fig:sample_complexity}
\end{wrapfigure}

A distinct advantage of ReAlign is its computational efficiency: it relies solely on low-order moments derived from unpaired data, avoiding the computational burden of iterative optimization. Here, we quantify the minimum sample size required for stable alignment.

We randomly sample $N$ unpaired images and texts from the pretraining corpus, with $N \in \{10^2, 5 \times 10^2, 10^3, 5 \times 10^3, 10^4, 5 \times 10^4, 10^5, 5 \times 10^5\}$. For each $N$, we estimate the ReAlign population means $\mu$ and second-order covariance shapes on these subsets. To ensure numerical stability in high dimensions, we apply standard shrinkage regularization to the covariance estimates. We align the embeddings using these estimates and measure the Modality Gap on a held-out test set. We report the mean $\pm$ standard deviation over $K=5$ independent trials.

As illustrated in Fig.~\ref{fig:sample_complexity}, the modality gap decreases rapidly with $N$ and stabilizes once $N$ exceeds approximately 10,000 samples. \fancynumber{1} Rapid Convergence: Unlike neural network training, which requires millions of gradient steps, moment-based estimation converges quickly via the Law of Large Numbers. Empirically, the performance plateaus after a moderate sample size ($N \approx 10^4$). \fancynumber{2} Implication: This confirms that ReAlign can be calibrated to new distributions on the fly with negligible computational cost, eliminating the need for massive paired datasets.

\subsection{Numerical Stability \& Complexity}

We conduct empirical experiments to validate the robustness and efficiency of our implementation.

Accumulating millions of high-dimensional vectors is prone to significant round-off errors. To quantify this, we compared the alignment residual $\|\mu_{err}\|_2$ (the distance between the computed mean and the true ground truth) using single-precision (\texttt{Float32}) versus double-precision (\texttt{Float64}) accumulators.
As shown in Fig.~\ref{fig:E(a)(b)}(a), utilizing \texttt{Float32} introduces an irreducible alignment error floor. Specifically, at $N=500,000$ samples, the error reaches $\approx 9.56 \times 10^{-8}$, which is non-negligible for precise manifold alignment. In contrast, our \texttt{Float64} implementation maintains precision near machine epsilon ($\approx 10^{-15}$), ensuring that the geometric alignment is limited only by statistical properties rather than numerical instability.

We analyze the time and space complexity of the ReAlign algorithm by varying the dataset size $N$ from 100k to 1M samples. \fancynumber{1} Time Complexity: As illustrated in Fig.~\ref{fig:E(a)(b)}(b), the processing time scales strictly linearly with data size ($O(N)$), increasing from 0.05s for 100k samples to 0.57s for 1M samples. \fancynumber{2} Space Complexity: Crucially, the peak memory usage remains constant at approximately 48.95 MB, regardless of the dataset size ($O(1)$). This confirms that ReAlign is highly scalable and capable of processing billions of tokens on standard hardware without memory bottlenecks.

\begin{figure*}[t] 
  \centering
  \includegraphics[width=0.92\linewidth]{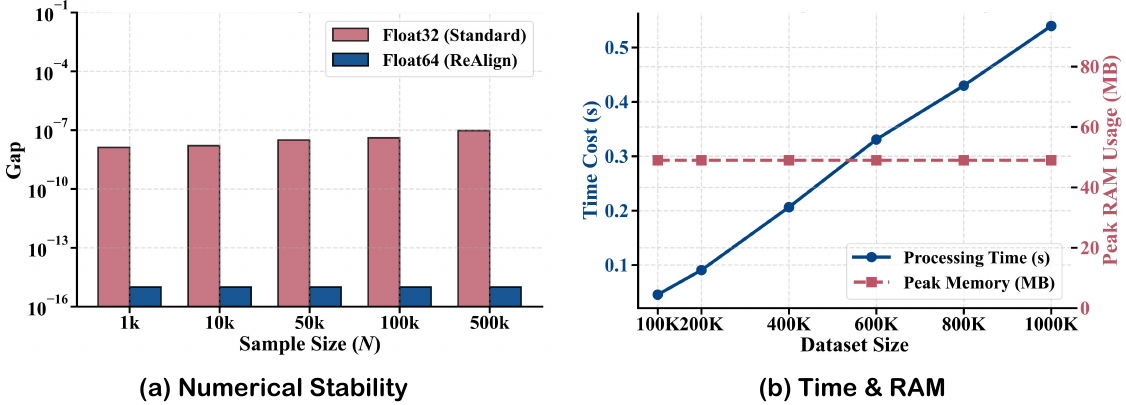}
  \caption{(a) Comparison of alignment residuals using Float32 vs. Float64 accumulators. (b) Trends of processing time ($O(N)$) and peak memory usage ($O(1)$) across dataset sizes.}
  \label{fig:E(a)(b)}
\end{figure*}

\subsection{Domain Mismatch Sensitivity}

\begin{wrapfigure}{r}{0.38\textwidth} 
  \centering
  \vspace{-30pt}
  \includegraphics[width=\linewidth]{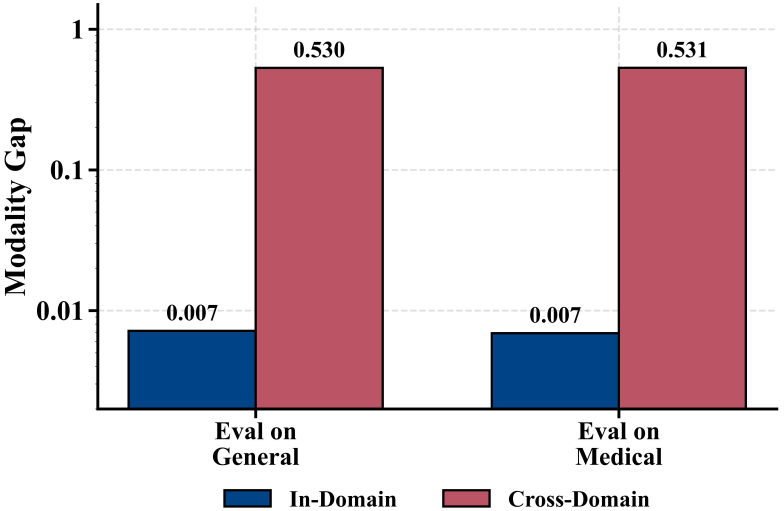}
  \caption{Comparison of modality gap performance under in-domain and cross-domain statistical alignment for General and Medical domains.}
  \vspace{-16pt}
  \label{fig:domain}

\end{wrapfigure}

We believe that applying general-domain statistics to specialized domains violates the specific distributional reality, leading to geometric misalignment. We quantify this effect using a cross-domain statistics transfer protocol.

We utilize two distinct distributions: \fancynumber{1} General: Bunny-Pretrain. \fancynumber{2} Medical: PubMedVision-Pretrain \cite{chen2024towards}, representing a specialized scientific domain.

\textbf{Protocol.} Let $\mathcal{D}_{\text{stat}}$ denote the source domain for statistical estimation, and $\mathcal{D}_{\text{eval}}$ denote the target domain for evaluation. We measure the modality gap on held-out data from $\mathcal{D}_{\text{eval}}$ under four settings: \fancynumber{1} General $\to$ General: In-domain baseline. \fancynumber{2} General $\to$ Medical: Cross-domain transfer. \fancynumber{3} Medical $\to$ Medical: In-domain alignment. \fancynumber{4} Medical $\to$ General: Reverse transfer.

Fig.~\ref{fig:domain} shows the results. \fancynumber{1} Transfer Degradation: Using General statistics to align Medical data results in a substantially larger modality gap compared to the in-domain baseline. \fancynumber{2} Specificity Requirement: Conversely, applying Medical statistics to General data also degrades performance. These results indicate that the covariance structure of the modality gap is domain-dependent. Therefore, domain-specific calibration is strictly necessary for precise geometric alignment.


\section{Failure Strategy Analysis: Blockwise Covariance Alignment}
\label{sec:training_free_substitution}

We present a training-free procedure that maps embeddings from one modality into the low-order statistical structure of another, using only first- and second-order moments estimated from large \emph{unpaired} samples. All operations are performed in the fixed reference frame of Sec.~\ref{subsec:decomposition_framework}, where the embedding space decomposes as
\(\mathbb{R}^d = U \oplus V\) with fixed projectors \(P_U, P_V\).
The procedure consists of three closed-form steps: \fancynumber{1} \textbf{Anchor align} to remove first-order bias effects; \fancynumber{2} \textbf{Geometry align} on \(U\) and \(V\) via whitening--coloring; and \fancynumber{3} \textbf{Centroid align} to rectify centroid shifts induced by the final spherical projection.

Throughout, we use:
\begin{equation}
    \mathrm{normalize}(z) = \frac{z}{\|z\|_2},
\end{equation}
and treat expectations/covariances as population quantities approximated by empirical statistics from large unpaired datasets.

\subsection{Step 1: Anchor Align}

We first align first-order statistics by translating each modality representation toward a shared anchor \(\mu^\star\) in the shared representation space and then projecting back to the unit sphere. Let \(e_a, e_b \in \mathbb{R}^d\) be unit-normalized embeddings for two modalities \(a\) and \(b\), with population means:
\begin{equation}
    \mu_a := \mathbb{E}[e_a], \qquad
    \mu_b := \mathbb{E}[e_b].
\end{equation}
In the fixed frame \((U,V)\), these means decompose into in-subspace (PMB) and orthogonal (COB) components.

To map modality \(a\) into the embedding distribution of modality \(b\), we choose the anchor \(\mu^\star=\mu_b\) and define
\begin{equation}
    \tilde e_a = \mathrm{normalize}(e_a - \mu_a + \mu_b), \qquad
    \tilde e_b = e_b.
\end{equation}
In this step, we translate \(e_a\) so that its mean matches \(\mu_b\), then re-project onto the unit sphere. The subsequent Step~2 then matches both modalities to a shared target geometry \((\Sigma_U^\star,\Sigma_V^\star)\).

\subsection{Step 2: Geometry Align}

Next, we match \emph{second-order} statistics on \(U\) and \(V\) separately.
All linear transforms are applied in Euclidean space, and the output is then projected back to the unit sphere. After Step~1, we decompose anchored embeddings using the fixed projectors \(P_U,P_V\) and estimate blockwise covariances:
\begin{equation}
\Sigma_{U,a}=\mathrm{Cov}(P_U \tilde e_a),\quad \Sigma_{V,a}=\mathrm{Cov}(P_V \tilde e_a).
\end{equation}
Analogously, define \(\Sigma_{U,b}\) and \(\Sigma_{V,b}\) for modality \(b\).

We set:
\begin{equation}
\Sigma_U^\star=\Sigma_{U,b},\quad \Sigma_V^\star=\Sigma_{V,b},
\end{equation}
so that modality \(a\) is mapped to match the covariance structure of modality \(b\) on both \(U\) and \(V\).
For symmetric substitution, \((\Sigma_U^\star,\Sigma_V^\star)\) can be chosen as any shared target geometry. Define the blockwise whitening--coloring transforms:
\begin{equation}
T_{U,a}=(\Sigma_U^\star)^{1/2}\,\Sigma_{U,a}^{-1/2},\quad
T_{V,a}=(\Sigma_V^\star)^{1/2}\,\Sigma_{V,a}^{-1/2}. 
\end{equation}
Using these, construct the substituted embedding by applying the two transforms on their respective blocks:
\begin{equation}
\hat e_a=\mathrm{normalize}\Big(T_{U,a}P_U\tilde e_a + T_{V,a}P_V\tilde e_a\Big). 
\end{equation}

\subsection{Step 3: Centroid Align}

Although Step 2 aligns the covariance structure, the final projection onto the unit sphere creates a non-linear distortion that inevitably shifts the population centroid away from the target anchor \(\mu^\star\). We term this phenomenon phantom drift. To strictly maintain first-order alignment, we explicitly estimate this drift and perform a final re-centering step. 

Let \(\hat{\mu}_a = \mathbb{E}[\hat e_a]\) be the drifted mean of the geometrically aligned embeddings from Step 2. We apply a final correction to pull the distribution back to the target anchor:
\begin{equation}
e_{final} = \mathrm{normalize}(\hat e_a - \hat{\mu}_a + \mu^\star).
\end{equation}
This ensures that the final substituted embeddings are both geometrically aligned in terms of covariance and accurately centered around the target semantic anchor.

\subsection{Failure Analysis: Why Fine-Grained Alignment Fails?}

\begin{figure*}[t] 
  \centering
  \includegraphics[width=0.92\linewidth]{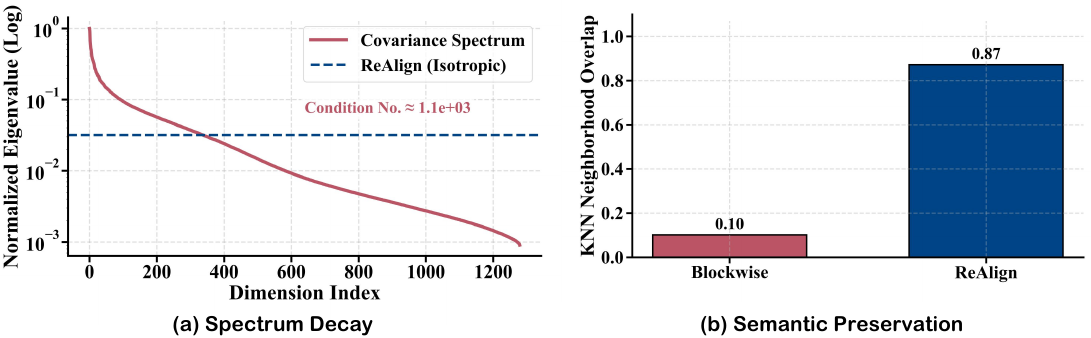}
  \caption{\textbf{Failure mechanism analysis of Blockwise Covariance Alignment.} (a) Spectrum Decay: The high condition number ($\approx 1.1 \times 10^3$) of text embeddings induces numerical instability during covariance inversion, amplifying tail noise. In contrast, ReAlign maintains stability via isotropic scaling. (b) Semantic Preservation: KNN neighborhood overlap rates reveal that Blockwise's aggressive whitening causes a collapse of the local semantic topology (retaining only 10\% overlap), whereas ReAlign effectively preserves 87\% of the semantic structure.}
  \label{fig:fail}
\end{figure*}

Although the Blockwise Covariance Alignment strategy theoretically offers a stricter geometric match by aligning full second-order moments within subspaces, empirical results indicate that it consistently underperforms the proposed ReAlign method across all benchmarks. We attribute this failure to two primary factors, supported by the quantitative analysis in Fig.~\ref{fig:fail}.

\begin{table*}[t]
\centering
\setlength{\tabcolsep}{2pt}
\scriptsize
\caption{\textbf{Performance comparison between BC Align and ReAlign}. The results demonstrate that ReVision significantly outperforms BC Align across all benchmarks. This confirms that the numerical instability and excessive distortion of the semantic manifold in BC Align lead to degraded model performance.}
\label{tab:multimodal_benchmarks}
\renewcommand{\arraystretch}{1.2}

\resizebox{1.0\textwidth}{!}{%
\begin{tabular}{lcccccccccccc}
\toprule
\multirow{2}{*}{\textbf{Method}}
& \multicolumn{4}{c}{\textbf{General}}
& \multicolumn{4}{c}{\textbf{Reasoning}}
& \multicolumn{3}{c}{\textbf{Hallucination}}
& \multirow{2}{*}{\textbf{Avg.}} \\
\cmidrule(lr){2-5} \cmidrule(lr){6-9} \cmidrule(lr){10-12}
& \texttt{MME} & \texttt{MMStar} & \texttt{SQA} & \texttt{RealWorldQA}
& \texttt{MMMU} & \texttt{MMMU-P} & \texttt{VisuLogic} & \texttt{LogicVista}
& \texttt{CRPE} & \texttt{POPE} & \texttt{HallBench} & \\
\midrule
BC Align & 76.12 & 34.33 & 76.42 & 45.36 & 30.69 & 27.95 & 24.40 & 22.60 & 80.93 & 71.08 & 46.27 & 48.74 \\
\rowcolor{graybg} \textbf{ReVision} & \textbf{79.65} & \textbf{36.13} & \textbf{76.71} & \textbf{47.97} & \textbf{31.51} & \textbf{28.39} & \textbf{27.70} & \textbf{22.82} & \textbf{81.78} & \textbf{72.53} & \textbf{46.58} & \textbf{50.16} \\

\bottomrule
\end{tabular}%
}
\end{table*}

\textbf{Estimation Variance and Numerical Instability.}
The ReAlign method relies on matching the \textit{global trace} $T = \text{tr}(\Sigma)$. Scalars are statistically robust and converge rapidly. In contrast, Blockwise Alignment requires estimating and inverting full covariance matrices.
As illustrated in Fig.~\ref{fig:fail}(a), the empirical covariance spectrum of the text embeddings exhibits a rapid decay with a high condition number ($\kappa \approx 1.10 \times 10^3$). 
Computing the whitening transform $\Sigma^{-1/2}$ on such an ill-conditioned matrix inadvertently amplifies noise in the tail dimensions (where eigenvalues $\lambda \to 0$) by orders of magnitude. This leads to numerical instability and exploding updates in minor feature directions, whereas ReAlign's isotropic scaling remains stable regardless of the spectral shape.

\textbf{Semantic Distortion via Aggressive Whitening.}
Geometric alignment operates on the assumption that the source and target manifolds share a topologically isomorphic structure. ReAlign adopts a conservative approach, isotropic scaling via $\tilde{e}_y=\mu_x+s(e_y-\mu_y)$., which preserves the relative angular distances between source embeddings. Blockwise whitening, however, applies an anisotropic rotation and scaling to force the covariance shapes to match exactly.

We quantified the impact of this transformation on semantic structure using $k$-Nearest Neighbor (KNN, $k=10$) overlap rates. As shown in Fig.~\ref{fig:fail}(b), this over-alignment causes a catastrophic collapse of the local semantic topology: Blockwise alignment retains only 10.1\% of the original semantic neighborhoods. In sharp contrast, ReAlign preserves 87.3\% of the local structure, ensuring that fine-grained semantic relationships remain intact for the LLM.

While Blockwise Covariance Alignment is geometrically more rigorous, ReAlign strikes a superior trade-off between geometric compatibility and semantic preservation. Its reliance on robust first-order and scalar second-order statistics provides a stable initialization that is crucial for scalable MLLM pretraining.


\section{The Long-Caption Paradox}
\label{app:long_caption_paradox}

Intuitively, utilizing longer, denser captions should theoretically provide richer semantic supervision, thereby enhancing cross-modal alignment. This expectation is particularly strong when employing advanced text encoders like LLM2CLIP, which are explicitly engineered to handle long contexts.

However, our empirical investigation reveals a counter-intuitive phenomenon which we term the \textbf{Long-Caption Paradox}. As demonstrated in Table.~\ref{tab:caption_length}, the model pretrained on the concise Bunny dataset consistently outperforms the model trained on the linguistically rich DenseFusion dataset. This performance gap persists despite both models utilizing the exact same ReAlign strategy and encoder. We attribute this paradox not to a single factor, but to structural constraints verified by our geometric analysis:\fancynumber{1} \textit{Representation Capacity}, \fancynumber{2} \textit{Diffuse Covariance}, and \fancynumber{3} \textit{Signal-to-Noise Ratio}.

\begin{table*}[t]
\centering
\setlength{\tabcolsep}{2pt}
\scriptsize
\caption{\textbf{Performance Comparison validating the Long-Caption Paradox.} Quantitative results demonstrate that ReVision consistently outperforms ReVision-Long across General, Reasoning, and Hallucination benchmarks. This performance gap confirms that for statistical modality alignment, the geometric compactness and high signal-to-noise ratio of short captions are more critical than the raw length of the textual description.}
\label{tab:caption_length}
\renewcommand{\arraystretch}{1.2}

\resizebox{1.0\textwidth}{!}{%
\begin{tabular}{lcccccccccccc}
\toprule
\multirow{2}{*}{\textbf{Method}}
& \multicolumn{4}{c}{\textbf{General}}
& \multicolumn{4}{c}{\textbf{Reasoning}}
& \multicolumn{3}{c}{\textbf{Hallucination}}
& \multirow{2}{*}{\textbf{Avg.}} \\
\cmidrule(lr){2-5} \cmidrule(lr){6-9} \cmidrule(lr){10-12}
& \textbf{MME} & \textbf{MMStar} & \textbf{SQA} & \textbf{RealWorldQA}
& \textbf{MMMU} & \textbf{MMMU-P} & \textbf{VisuLogic} & \textbf{LogicVista}
& \textbf{CRPE} & \textbf{POPE} & \textbf{HallBench} & \\
\midrule
ReVision-Long & 75.74 & 33.40 & 74.84 & 46.41 & 30.57 & 27.26 & 24.70 & \textbf{25.28} & 80.42 & 71.47 & 45.95 & 48.73\\
\rowcolor{graybg} \textbf{ReVision} & \textbf{79.65} & \textbf{36.13} & \textbf{76.71} & \textbf{47.97} & \textbf{31.51} & \textbf{28.39} & \textbf{27.70} & 22.82 & \textbf{81.78} & \textbf{72.53} & \textbf{46.58} & \textbf{50.16} \\
\bottomrule
\end{tabular}%
}
\end{table*}

\begin{figure*}[t] 
  \centering
  \includegraphics[width=0.9\linewidth]{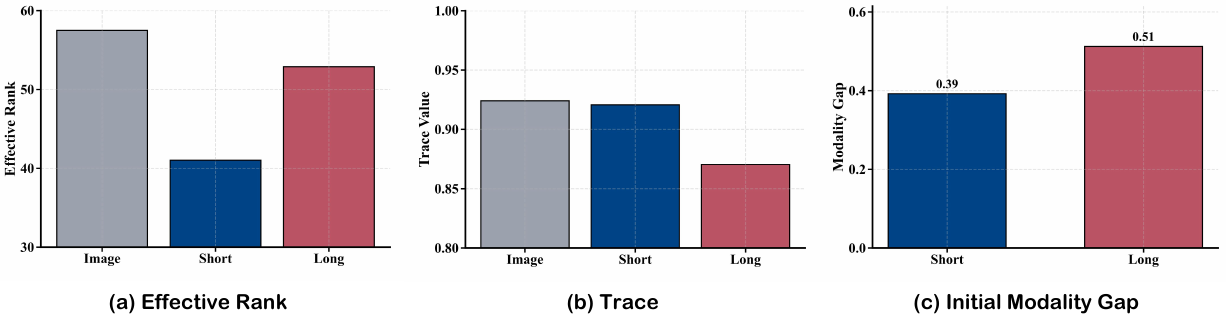}
  \caption{\textbf{Geometric Analysis of the Long-Caption Paradox.} (a) Effective Rank: Measurements reveal that Long captions exhibit a high effective rank ($\approx 52.9$) similar to the visual modality ($\approx 57.5$). This indicates a diffuse, high-entropy covariance structure that is difficult to align. In contrast, Short captions ($\approx 41.0$) function as a compact, low-rank approximation of the visual content, offering greater statistical stability. (c) Initial Modality Gap: The inclusion of non-visual linguistic noise in long captions acts as a disturbing force, significantly widening the initial modality gap ($\|\Delta\mu\| \approx 0.51$) by approximately 30\% compared to concise captions ($\|\Delta\mu\| \approx 0.39$).}
  \label{fig:long_caption}
\end{figure*}

\subsection{Truncation-Induced Supervision Mismatch} 

Even with advanced encoders like LLM2CLIP that extend the context window, there remains a hard physical limit on the input sequence length. Unlike the visual encoder, which processes the image holistically, the text encoder operates within a fixed context window. When a dense caption $T$ exceeds this limit, the trailing semantic content is inevitably truncated and discarded before encoding. This truncation creates a fundamental discrepancy between the modalities. The visual embedding $v$ encodes the global visual information, whereas the textual embedding $z$ only captures the truncated prefix of the caption. Consequently, the model attempts to align the full image to an incomplete textual description, as the details described in the truncated tail are absent from the feature space.

\subsection{Diffuse Covariance Structure}
From the perspective of ReAlign, long captions introduce specific geometric challenges. We quantify this using the effective rank to measure the complexity and compactness of the embedding manifolds. As shown in Fig.~\ref{fig:long_caption}, our measurements reveal that visual embeddings possess a high Effective Rank ($\approx 57.5$), reflecting their rich detail. Long captions attempt to mimic this complexity, exhibiting a similarly high rank ($\approx 52.9$). However, this high rank implies a diffuse covariance structure, a high-dimensional cloud with high entropy. In contrast, Short captions exhibit a significantly lower Effective Rank ($\approx 41.0$). This suggests that short captions act as a low-rank approximation of the visual content, filtering out background noise and retaining only the most structurally salient principal components. Aligning a compact manifold to the visual manifold is statistically more stable than aligning a diffuse, high-entropy cloud.

\begin{figure*}[t] 
  \centering
  \includegraphics[width=0.92\linewidth]{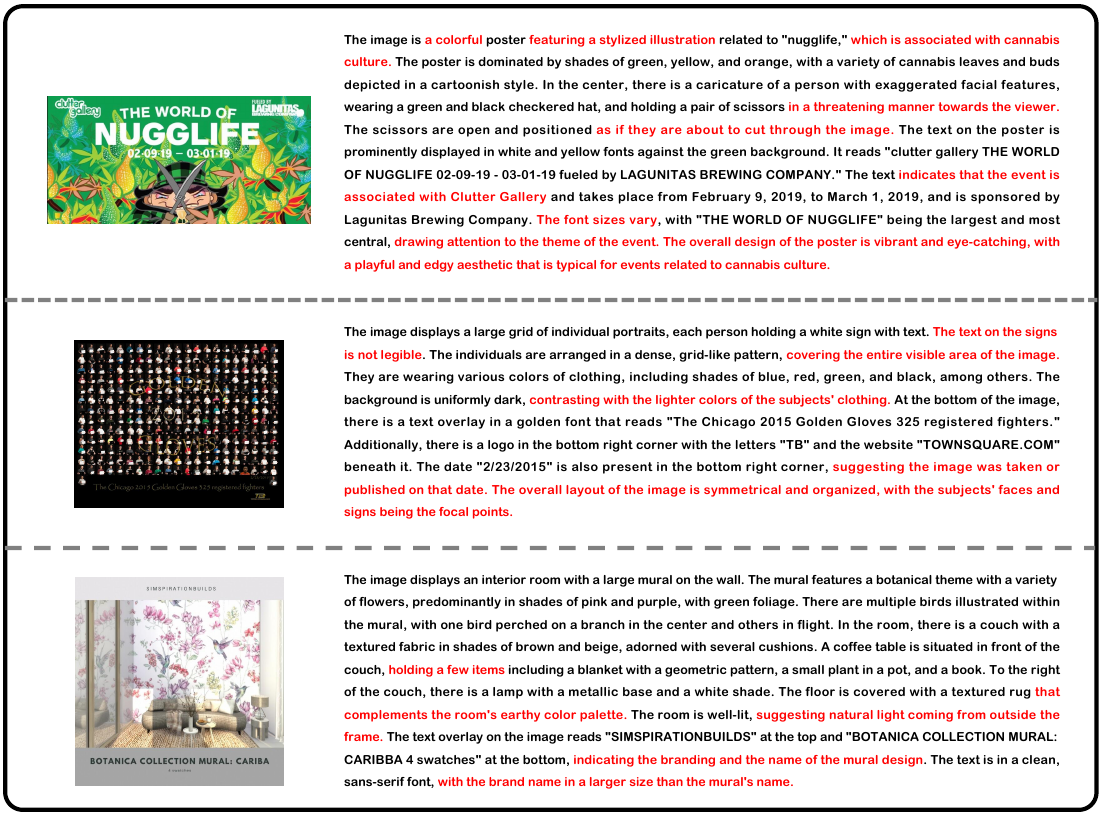}
  \caption{Visualization of Linguistic Noise in Dense Captions. We highlight non-visual segments (marked in red) within long captions, such as abstract inferences, contextual associations, and subjective interpretations. These tokens lack direct visual grounding and geometrically act as noise vectors, pulling the semantic centroid away from the true visual anchor.}
  \label{fig:linguistic_noise}
\end{figure*}

\subsection{Linguistic Noise \& Modality Gap}
Finally, we distinguish between Linguistic information and Visual Information. A larger token count does not equate to higher visual utility. Dense captions often contain non-visual context (The atmosphere was tense) or abstract inferences, as visualized in Fig.~\ref{fig:linguistic_noise}. Geometrically, these tokens act as noise vectors that pull the semantic centroid away from the visual anchor. This is empirically verified by the initial modality gap. As shown in our experiments, the long-caption dataset exhibits a significantly larger gap ($\|\Delta\mu\| \approx 0.51$) compared to the Short-Caption dataset ($\|\Delta\mu\| \approx 0.39$). 

\textbf{Conclusion:} The extra linguistic information in long captions effectively acts as a disturbing force, increasing the modality gap by $\approx 30\%$. For statistical modality alignment, the visual grounding and geometric compactness of the text distribution are more critical than raw length.

\section{Experiments Setting} \label{app:setting}

\subsection{CLIP Pretraining Setting}

To empirically validate theoretical hypotheses, we established a controlled pretraining environment utilizing a TinyCLIP architecture comprising a 40 million parameter ViT B/32 vision encoder and a 19 million parameter Transformer text encoder. The model was trained on 2 million randomly sampled image-text pairs from LAION 400M with 224 by 224 resolution inputs using the AdamW optimizer with a weight decay of 0.0001 and a cosine learning rate schedule peaking at 0.0001 with a warmup ratio of 0.1. The training process was distributed across 8 NVIDIA A6000 GPUs with a global batch size of 4096. Distinct from standard protocols, we integrated a custom online geometric monitoring system that utilizes a dedicated held-out probe set of 200,000 samples to perform high-frequency spectral analysis every 20 steps. This protocol dynamically updates the reference subspace basis until step 800 before freezing it for stable drift measurement and employs 16 accumulated batches to approximate the Fisher Information Matrix for curvature analysis while tracking the top 128 eigen components to verify spectral decay. All linear algebraic computations are executed in Float32 precision to ensure numerical stability against the BF16 training backdrop.

\subsection{MLLM Training Setting}

We employ Llama-3-8B-Instruction as the LLM backbone, connected to the input modalities via a two-layer MLP projector with GELU activation. Our core design utilizes aligned text representations as pseudo-visual tokens. These text-derived embeddings are mapped directly into the LLM's feature space through the MLP. Training Pipeline. The training process consists of two stages: \fancynumber{1} Stage 1:  Modality Substitution Pretraining. We train the projector for 1 epoch on the filtered Bunny-1M dataset with a learning rate of \(5 \times 10^{-4}\), while keeping the LLM parameters frozen. \fancynumber{2} Stage 2: Visual Instruction Tuning. We perform full-parameter fine-tuning on the InternVL-Chat-V1.2 dataset for 1 epoch. The projector is initialized with weights from Stage 1, and the learning rate is reduced to 1e-5. The experiments were conducted on 8 NVIDIA H200 GPUs. With a total data scale of approximately 2.2M samples, the complete training pipeline was finished in 12 hours.

\subsection{Eval Setting}

We evaluate ReVision across three primary dimensions using a comprehensive suite of benchmarks: General Perception (MME test \cite{fu2025mme}, MMStar \cite{chen2024we}, ScienceQA-image dev\&test \cite{lu2022learn}, and RealWorldQA), Complex Reasoning (MMMU validation single-image \cite{yue2024mmmu}, MMMU-Pro single-image \cite{yue2025mmmu}, VisuLogic train \cite{xu2025visulogic}, and LogicVista \cite{xiao2024logicvista}), and Hallucination Assessment (CRPE \cite{wang2024all}, POPE \cite{li2023evaluating}, and HallusionBench \cite{guan2024hallusionbench}). For all benchmarks, we report accuracy ($acc$) as the unified metric to ensure a consistent and rigorous comparison.


\subsection{Cost Analysis}

We quantify the data acquisition cost using the formula $C_{\text{total}} = P_{\text{in}} \cdot T_{\text{in}} + P_{\text{out}} \cdot T_{\text{out}}$, where $C_{\text{total}}$ denotes the total cost, and $T$ and $P$ represent the token count and price per million tokens, respectively. To ensure a standardized comparison, we assume unified pricing based on GPT-5 APIs for all methods, with rates set at $P_{\text{in}} = \$1.25$ and $P_{\text{out}} = \$10.00$ per million tokens. Under this setting, the methods exhibit distinct cost profiles: Unicorn incurs the highest expense of \$1893.27 (17.50M input, 187.14M output), whereas the standard ReVision (1M) is the most economical at \$176.10 (11.64M input, 16.15M output). Scaling to ReVision-2M brings the cost to \$352.20 (23.28M input, 32.30M output), which remains notably lower than the paired w/. Image baseline cost of \$476.05 (221.64M input, 16.15M output). For clarity, all costs are normalized relative to the w/. Image baseline (1.0).


\section{More MLLM Experiments}
\label{app:more_mllm_experiments}

In this section, we provide additional MLLM experiments to further analyze the effectiveness and scalability of ReVision. We first ablate the three components of ReAlign, then study the role of Modality Substitution Pretraining, and finally evaluate scaling behavior with respect to Stage-2 visual instruction data and LLM backbone size.

\subsection{Ablation of ReAlign Components}
\label{app:realign_component_ablation}

To verify the contribution of each step in ReAlign, we conduct a component-wise ablation by progressively adding Anchor Alignment, Trace Alignment, and Centroid Alignment. As shown in Table~\ref{tab:realign_component_ablation}, Anchor Alignment alone already provides a strong baseline by correcting the first-order mean mismatch. Adding Trace Alignment further improves the average score by matching the global residual energy scale while preserving the learned spectral structure. Finally, Centroid Alignment brings additional gains by correcting the centroid drift induced by spherical normalization. These results support the design of ReAlign as a three-stage calibration procedure.

\begin{table*}[t]
\centering
\setlength{\tabcolsep}{2pt}
\scriptsize
\caption{Ablation of ReAlign components. Step 1 denotes Anchor Alignment, Step 2 denotes Trace Alignment, and the full ReAlign additionally includes Centroid Alignment.}
\label{tab:realign_component_ablation}
\renewcommand{\arraystretch}{1.2}
\resizebox{\textwidth}{!}{%
\begin{tabular}{lcccccccccccc}
\toprule
\multirow{2}{*}{\textbf{Method}}
& \multicolumn{4}{c}{\textbf{General}}
& \multicolumn{4}{c}{\textbf{Reasoning}}
& \multicolumn{3}{c}{\textbf{Hallucination}}
& \multirow{2}{*}{\textbf{Avg. $\uparrow$}} \\
\cmidrule(lr){2-5} \cmidrule(lr){6-9} \cmidrule(lr){10-12}
& \texttt{MME} & \texttt{MMStar} & \texttt{SQA} & \texttt{RealWorldQA}
& \texttt{MMMU} & \texttt{MMMU-P} & \texttt{VisuLogic} & \texttt{LogicVista}
& \texttt{CRPE} & \texttt{POPE} & \texttt{HallBench} & \\
\midrule
w/. Step 1 & 74.86 & 33.18 & 73.93 & 42.64 & 28.48 & 25.61 & 22.94 & 18.86 & 78.93 & 70.82 & 43.24 & 46.68 \\
w/. Step 1 \& 2 & 78.22 & 35.71 & 76.21 & 45.16 & 30.05 & 26.09 & 26.05 & 22.16 & 80.37 & 71.18 & 44.14 & 48.67 \\
\hdashline
\rowcolor{graybg} \textbf{ReAlign / ReVision} & \textbf{79.65} & \textbf{36.13} & \textbf{76.71} & \textbf{47.97} & \textbf{31.51} & \textbf{28.39} & \textbf{27.70} & \textbf{22.82} & \textbf{81.78} & \textbf{72.53} & \textbf{46.58} & \textbf{50.16} \\
\bottomrule
\end{tabular}
}
\end{table*}

\subsection{Effect of Modality Substitution Pretraining}
\label{app:stage1_ablation}

We further ablate the first stage of ReVision to examine whether Modality Substitution Pretraining is necessary. In the w/o Stage 1 setting, the model skips pseudo-visual pretraining and directly proceeds to visual instruction tuning. As shown in Table~\ref{tab:stage1_ablation}, removing Stage 1 leads to a substantial performance drop across all benchmarks, decreasing the average score from \(50.16\) to \(43.58\). This confirms that Stage 1 is not merely a generic projector warm-up; instead, it provides a visually compatible semantic interface before real-image instruction tuning.

\begin{table*}[t]
\centering
\setlength{\tabcolsep}{2pt}
\scriptsize
\caption{Ablation of Modality Substitution Pretraining. Removing Stage 1 consistently degrades performance across all benchmarks.}
\label{tab:stage1_ablation}
\renewcommand{\arraystretch}{1.2}
\resizebox{\textwidth}{!}{%
\begin{tabular}{lcccccccccccc}
\toprule
\multirow{2}{*}{\textbf{Method}}
& \multicolumn{4}{c}{\textbf{General}}
& \multicolumn{4}{c}{\textbf{Reasoning}}
& \multicolumn{3}{c}{\textbf{Hallucination}}
& \multirow{2}{*}{\textbf{Avg. $\uparrow$}} \\
\cmidrule(lr){2-5} \cmidrule(lr){6-9} \cmidrule(lr){10-12}
& \texttt{MME} & \texttt{MMStar} & \texttt{SQA} & \texttt{RealWorldQA}
& \texttt{MMMU} & \texttt{MMMU-P} & \texttt{VisuLogic} & \texttt{LogicVista}
& \texttt{CRPE} & \texttt{POPE} & \texttt{HallBench} & \\
\midrule
w/o Stage 1 & 66.80 & 29.70 & 71.90 & 39.10 & 25.60 & 23.40 & 20.10 & 16.90 & 75.80 & 68.90 & 41.20 & 43.58 \\
\hdashline
\rowcolor{graybg} \textbf{ReVision} & \textbf{79.65} & \textbf{36.13} & \textbf{76.71} & \textbf{47.97} & \textbf{31.51} & \textbf{28.39} & \textbf{27.70} & \textbf{22.82} & \textbf{81.78} & \textbf{72.53} & \textbf{46.58} & \textbf{50.16} \\
\bottomrule
\end{tabular}
}
\end{table*}

\subsection{Scaling Visual Instruction Tuning Data}
\label{app:sft_scaling}

We study how the amount of Stage-2 visual instruction tuning data affects ReVision. We keep Stage 1 fixed and vary only the amount of Stage-2 SFT data. As shown in Table~\ref{tab:sft_scaling}, increasing the SFT data from \(30\%\) to \(60\%\) and then to \(100\%\) consistently improves performance. This trend indicates that Stage 1 establishes a useful cross-modal semantic interface, while Stage 2 progressively injects fine-grained real-image supervision and instruction-following ability.

\begin{table*}[t]
\centering
\setlength{\tabcolsep}{2pt}
\scriptsize
\caption{Scaling Stage-2 visual instruction tuning data. Stage 1 is fixed, and only the amount of Stage-2 SFT data is varied.}
\label{tab:sft_scaling}
\renewcommand{\arraystretch}{1.2}
\resizebox{\textwidth}{!}{%
\begin{tabular}{lcccccccccccc}
\toprule
\multirow{2}{*}{\textbf{Stage-2 Data}}
& \multicolumn{4}{c}{\textbf{General}}
& \multicolumn{4}{c}{\textbf{Reasoning}}
& \multicolumn{3}{c}{\textbf{Hallucination}}
& \multirow{2}{*}{\textbf{Avg. $\uparrow$}} \\
\cmidrule(lr){2-5} \cmidrule(lr){6-9} \cmidrule(lr){10-12}
& \texttt{MME} & \texttt{MMStar} & \texttt{SQA} & \texttt{RealWorldQA}
& \texttt{MMMU} & \texttt{MMMU-P} & \texttt{VisuLogic} & \texttt{LogicVista}
& \texttt{CRPE} & \texttt{POPE} & \texttt{HallBench} & \\
\midrule
30\% & 68.53 & 30.60 & 70.90 & 41.88 & 28.67 & 25.33 & 23.60 & 17.17 & 74.65 & 66.28 & 40.35 & 44.36 \\
60\% & 75.99 & 34.67 & 75.06 & 45.36 & 30.12 & 27.47 & 25.10 & 20.06 & 78.57 & 70.53 & 44.39 & 47.94 \\
\hdashline
\rowcolor{graybg} \textbf{100\%} & \textbf{79.65} & \textbf{36.13} & \textbf{76.71} & \textbf{47.97} & \textbf{31.51} & \textbf{28.39} & \textbf{27.70} & \textbf{22.82} & \textbf{81.78} & \textbf{72.53} & \textbf{46.58} & \textbf{50.16} \\
\bottomrule
\end{tabular}
}
\end{table*}

\subsection{Scaling the LLM Backbone}
\label{app:backbone_scaling}

We evaluate whether ReVision benefits from scaling the LLM backbone. We replace only the LLM backbone from Llama-3-8B-Instruct to Llama-3-70B-Instruct, while keeping the visual/text encoder, ReAlign operator, Stage-1 corpus, Stage-2 SFT dataset, and evaluation suite unchanged. As shown in Table~\ref{tab:backbone_scaling}, scaling the backbone improves the average score from \(50.16\) to \(52.33\). The gains are especially visible on reasoning-intensive benchmarks, suggesting that once the geometric pseudo-visual interface is established, a larger LLM can better exploit the injected semantic supervision.

\begin{table*}[t]
\centering
\setlength{\tabcolsep}{2pt}
\scriptsize
\caption{Scaling the LLM backbone under the same ReVision pipeline. Only the LLM backbone is changed, while all other components and data settings remain fixed.}
\label{tab:backbone_scaling}
\renewcommand{\arraystretch}{1.2}
\resizebox{\textwidth}{!}{%
\begin{tabular}{lcccccccccccc}
\toprule
\multirow{2}{*}{\textbf{Backbone}}
& \multicolumn{4}{c}{\textbf{General}}
& \multicolumn{4}{c}{\textbf{Reasoning}}
& \multicolumn{3}{c}{\textbf{Hallucination}}
& \multirow{2}{*}{\textbf{Avg. $\uparrow$}} \\
\cmidrule(lr){2-5} \cmidrule(lr){6-9} \cmidrule(lr){10-12}
& \texttt{MME} & \texttt{MMStar} & \texttt{SQA} & \texttt{RealWorldQA}
& \texttt{MMMU} & \texttt{MMMU-P} & \texttt{VisuLogic} & \texttt{LogicVista}
& \texttt{CRPE} & \texttt{POPE} & \texttt{HallBench} & \\
\midrule
Llama-3-8B-Instruct & 79.65 & 36.13 & 76.71 & 47.97 & 31.51 & 28.39 & 27.70 & 22.82 & 81.78 & 72.53 & 46.58 & 50.16 \\
\hdashline
\rowcolor{graybg} \textbf{Llama-3-70B-Instruct} & \textbf{80.27} & \textbf{37.84} & \textbf{78.92} & \textbf{49.16} & \textbf{34.99} & \textbf{31.27} & \textbf{31.54} & \textbf{26.11} & \textbf{82.52} & \textbf{73.89} & \textbf{49.15} & \textbf{52.33} \\
\bottomrule
\end{tabular}
}
\end{table*}

\section{Qualitative Analysis} \label{app:qualitative}

To intuitively understand the improvements brought by ReVision, we conduct a comprehensive evaluation across three distinct cognitive levels: General Visual Perception, Domain-Specific Knowledge, and Logical Reasoning. As visualized in Fig.~\ref{fig:examples}, the model demonstrates exceptional capability in handling complex Abstract and Spatial Reasoning tasks. Specifically, in the matrix pattern completion case, the model moves beyond simple texture matching to identify the underlying geometric progression rules. This capacity for mental simulation is further evidenced in the geometry folding problem, where the model successfully performs mental rotation to reconstruct a 3D cube from a 2D net, accurately predicting the spatial adjacency of symbols. Beyond abstract logic, the model exhibits robust Fine-Grained Perception and Knowledge Integration. In the domain of data analysis, it accurately interprets the intersection points of two trend lines in a statistical chart, extracting precise numerical values rather than merely describing general trends. Furthermore, the model effectively grounds visual signals into specific world knowledge. This is demonstrated in the geospatial recognition scenario, where it combines visual map boundaries with geographical facts to identify the easternmost state, and in the physical common sense task, where it correctly deduces magnetic attraction based on the orientation of poles.

\begin{figure*}[t] 
  \centering
  \includegraphics[width=0.9\linewidth]{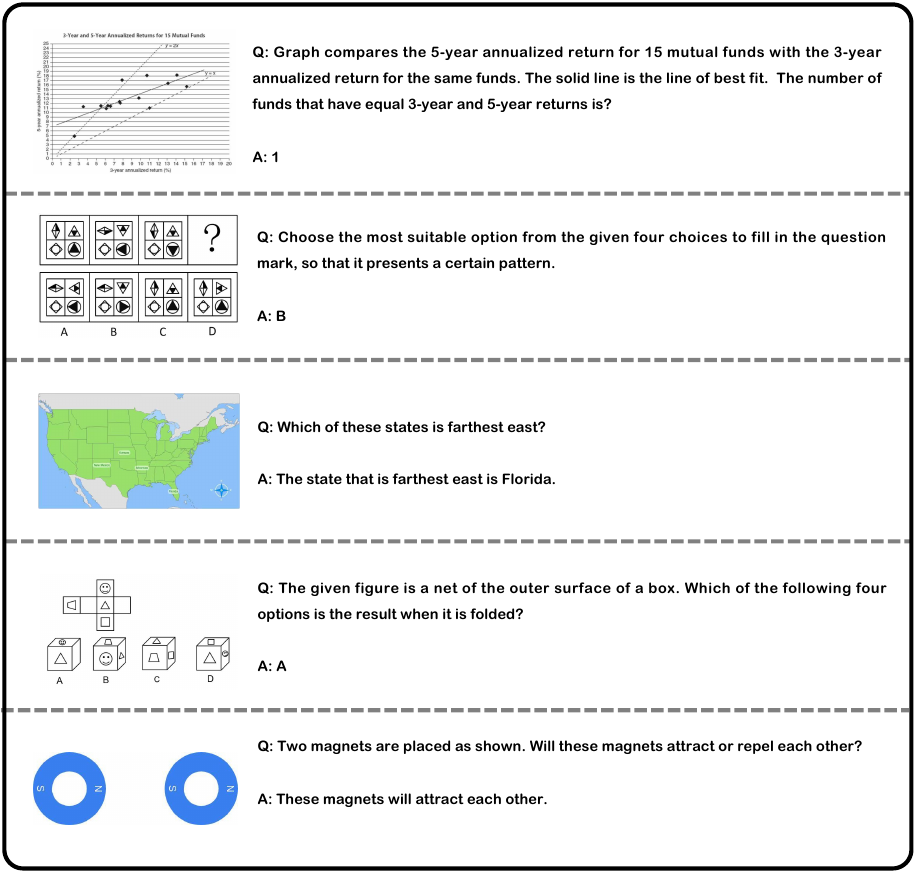}
  \caption{\textbf{Qualitative analysis examples of ReVision.} This figure shows the model's visual perception ability across various complex tasks.}
  \label{fig:examples}
\end{figure*}

\section{Related Work} \label{app:related_work}

The Modality Gap refers to the systematic distributional distance between representations of paired data from distinct modalities within a shared embedding space, despite the theoretical expectation that semantically identical pairs should align. Early empirical research first identified this geometric anomaly as a cone effect \cite{liang2022mind}, observing that embedding vectors tend to occupy a narrow cone rather than spanning the full hypersphere \cite{wang2020understanding}. The C$^3$ framework \cite{zhang2024connect} provided the first formal description of this structure, characterizing the gap as a superposition of a constant orthogonal displacement and random alignment noise. While these frameworks offer a foundational understanding, they predominantly rely on the assumption that the noise term is isotropic. Our work challenges this simplification, demonstrating that the residual noise exhibits high anisotropy, thereby necessitating more precise second-moment modeling. The absence of this theoretical perspective has directly limited recent explorations into training MLLMs \cite{liu2023visual,bai2025qwen2,wang2025internvl3,wu2024deepseek,yu2025minicpm,alayrac2022flamingo} using pure text. Unicorn \cite{yu2025unicorn} pioneered the use of the modality gap to convert text representation into the pseudo-visual representation. However, due to their reliance on simple mean shifting, which implies an isotropic assumption, the synthesized representation fails to match the complex geometric shape of the real visual representation. This further highlights the necessity of constructing a novel training paradigm based on precise covariance alignment.

\section{Broader Impact} \label{app:future}

Our work advances the paradigm of Data-Efficient AI, specifically by decoupling the acquisition of visual semantics from the reliance on massive paired datasets. By demonstrating that the heavy lifting of knowledge injection can be offloaded to massive unpaired text during the Pretraining stage, while reserving limited real images solely for the SFT stage, we fundamentally alter the resource landscape for training MLLMs. This paradigm shift has profound implications in several key areas:

\textbf{Democratization of Multimodal Research:} The prohibitive cost of collecting and cleaning billion-scale image-text pairs has historically concentrated MLLM development within a few well-resourced institutions. By shifting the data requirement to abundant unpaired text for the bulk of training, our approach significantly lowers the barrier to entry, enabling academic labs and smaller research groups to train high-performance models from scratch.

\textbf{Expansion to Low-Resource Domains:} In specialized fields (e.g., medical imaging, minority languages, or technical diagrams), paired data is scarce, yet textual knowledge is often available. Our ReVision paradigm allows models to learn visual concepts primarily through domain-specific text corpora during pretraining, requiring only a handful of examples for SFT. This opens new avenues for deploying MLLMs in domains where data scarcity previously made it impossible.

\textbf{Mitigation of Copyright and Privacy Risks:} Large-scale web-scraped image datasets often contain copyrighted artwork and sensitive personal identifiable information (PII). By minimizing the reliance on raw images during the data-hungry pretraining phase and relying instead on text, our method offers a potential pathway to reduce legal and ethical risks associated with dataset curation.

\end{document}